%% file: main.tex
\documentclass{article}

\usepackage{microtype}
\usepackage{graphicx}
\usepackage{subcaption}
\usepackage{booktabs} %

\usepackage{hyperref}

\usepackage[preprint]{icml2026}

\usepackage{amsmath}
\usepackage{amssymb}

\input{preamble.tex}

\mathtoolsset{showonlyrefs}
\input{certified.tex} %
\theoremstyle{plain}
\newtheorem{theorem}{Theorem}[section]

\theoremstyle{definition}
\newtheorem{definition}[theorem]{Definition}
\newtheorem{assumption}[theorem]{Assumption}
\theoremstyle{remark}
\newtheorem{remark}{Remark}

\usepackage[textsize=tiny]{todonotes}

\icmltitlerunning{Error Feedback Algorithms in Distributed Optimization}

\begin{document}

\twocolumn[

  \icmltitle{A Tight Theory of Error Feedback Algorithms in Distributed Optimization}

  \icmlsetsymbol{equal}{*}

  \begin{icmlauthorlist}
    \icmlauthor{Daniel Berg Thomsen}{ens,x}
    \icmlauthor{Adrien Taylor}{ens}
    \icmlauthor{Aymeric Dieuleveut}{x}
  \end{icmlauthorlist}

  \icmlaffiliation{ens}{Inria, D.I. ENS, CNRS, PSL Research University, Paris, France}
  \icmlaffiliation{x}{CMAP, École Polytechnique, Institut Polytechnique de Paris, Palaiseau, France}

  \icmlcorrespondingauthor{Daniel Berg Thomsen}{daniel.berg-thomsen@inria.fr}

  \icmlkeywords{Machine Learning, ICML}

  \vskip 0.3in
]

\printAffiliationsAndNotice{}  %

\begin{abstract}
  Communication costs are a major bottleneck in distributed learning and first-order optimization. A
  common approach to alleviate this issue is to compress the gradient information exchanged between
  agents. However, such compression typically degrades the convergence guarantees of gradient-based
  methods. \emph{Error feedback} mechanisms provide a simple and computationally cheap remedy for this
  issue, but numerous variants have been proposed, and their relative performance remains poorly
  understood. This paper provides \emph{tight} convergence analyses for two of the main error-feedback algorithms from the literature,
  the classic Error Feedback method ($\EF$) and Error Feedback 21 ($\EFtw$), by identifying optimal step-size choices
  and constructing optimal Lyapunov functions tailored to each method. The results hold independently
  of the number of agents and recover the known best guarantees possible in the single-agent regime.
\end{abstract}

\section{Introduction}
The trend toward larger model usage in machine learning has made training in distributed environments a practical necessity.
In many applications, including federated learning, data are partitioned across \(n\) agents
and a central server coordinates the process~\cite{mcmahan_communication-efficient_2017,kairouz_advances_2019}.
Consider the finite-sum problem
\begin{equation}\label{eq:main_problem}
  \min_{x \in \R^d} \left[f(x) \coloneqq \frac{1}{n} \sum_{i=1}^n f^{(i)}(x)\right],
\end{equation}
where only agent \(i\) has access to first-order information about \(f^{(i)}\).
Communication follows a star topology: at each round, agents send messages to the server, the
server aggregates them into a new iterate and broadcasts this iterate to all agents.
Because many agents communicate concurrently, the agents-to-server traffic is often the dominant
bottleneck~\cite{seide_1-bit_2014,chilimbi_project_2014,strom_scalable_2015}, especially for large models seen in deep learning.

A standard approach to mitigate communication costs is to reduce the frequency of
communication~\cite{mcmahan_communication-efficient_2017,karimireddy_scaffold_2020,mishchenko2022proxskip} and/or to
transmit \emph{compressed} messages.
Compression may be applied on the uplink (agents$\to$server)~\cite{seide_1-bit_2014,alistarh_qsgd_2017,richtarik_ef21_2021}
and/or downlink (server$\to$agents)~\cite{harrane_reducing_2018,philippenko_artemis_2020,gorbunov_linearly_2020}; the focus here is on uplink compression.
Common compressors include low-precision quantization~\cite{alistarh_qsgd_2017}, sparsification (e.g., Top-$K$,
~\citealt{alistarh_convergence_2018}), and using random projections~\cite{vempala2005random} like in
\textsc{Sketched-SGD}~\cite{ivkin_communication-efficient_2019}.

To analyze algorithms independently of a particular compressor, one typically assumes general
properties of a (possibly random) compression operator \(\C\).
Compression is modeled as a family of deterministic maps indexed by a seed \(\omega\),
written \(\C(\cdot; \omega): \mathbb{R}^d \to \mathbb{R}^d\). Unless stated otherwise, new
seeds are drawn for the compression calls at each iteration; at a given iteration, the
same seed may be used by several agents.
When needed, \(\omega_k^{(i)}\) denotes
the seed used by agent \(i\) in the compression step at iteration \(k\), and
\(\omega_k \coloneqq (\omega_k^{(i)})_{i=1}^n\); expectations are taken with respect to the
relevant seeds.
A classical assumption is unbiasedness, i.e., \(\E_{\omega}[\C(x; \omega)] = x\) for all \(x \in \mathbb{R}^d\).
Another widely used assumption is \emph{contractiveness}:
\begin{assumption}[Contractive compressor]\label{as:compression}
  The compression operator \(\C(\cdot; \omega)\) is such that, for some \(\epsilon \in [0, 1)\),
  \begin{equation}
    \text{for all} \quad x \in \mathbb{R}^d, \quad \E_{\omega}\left[\|x - \C(x; \omega)\|^2\right] \leq \epsilon \|x\|^2.
  \end{equation}
\end{assumption}

A natural baseline is \emph{Compressed Gradient Descent} ($\CGD$):
\begin{equation}
  x_{k+1} \coloneqq x_k - \frac{\eta}{n} \sum_{i=1}^n \C(\nabla f^{(i)}(x_k); \omega_k^{(i)}),
  \tag{\CGD}\label{eq:cgd_update}
\end{equation}
where \(\eta > 0\) is the step size.
However, $\CGD$ generally fails to converge in the multi-agent setting, as shown for instance when
\(\C\) is biased in \citet{beznosikov_biased_2020}.

To mitigate the effects of compression, each agent can use \emph{Error Feedback} ($\EF$). The 
classic mechanism~\cite{seide_1-bit_2014,karimireddy_error_2019} keeps track of past compression 
errors and “reinjects” them into later messages, as described in \Cref{alg:ef}.

Another challenge in distributed environments is heterogeneity across agents: the local 
objectives $f^{(i)}$ can differ substantially, for example due to mismatched data distributions 
or uneven scaling induced by non-uniform data partitioning. Such heterogeneity can obstruct 
convergence. In the strongly convex setting, \citet{beznosikov_biased_2020} 
prove linear convergence of distributed \(\EF\) with biased compression, but only under the 
additional assumption that all local objectives share the same minimizer—i.e., the 
interpolation regime~\cite{ma2018power,vaswani2019fast}.

Motivated by the need to directly handle heterogeneity, \citet{richtarik_ef21_2021} proposed
\emph{Error Feedback 21}
(\(\EFtw\)), given in \Cref{alg:ef21}.
In \(\EFtw\), each agent maintains an estimator \(d_k^{(i)}\) of the local gradient.
At iteration \(k\), the server first updates \(x_{k+1}\) using the current estimators
\(d_k^{(i)}\). The agents then communicate compressed differences
\(\C(\nabla f^{(i)}(x_{k+1}) - d_k^{(i)}; \omega_k^{(i)})\), which are used to update
the estimators for the next iteration. This tracking step is argued to be more stable
around the global optimum of the sum total objective~\eqref{eq:main_problem}.

By now, a substantial literature has been devoted to the analysis of both algorithms under different assumptions on the
communication model~\cite{koloskova_decentralized_2019,philippenko_preserved_2021}, the compression
operator~\cite{alistarh_qsgd_2017,stich_sparsified_2018,beznosikov_biased_2020}, and the objective
functions~\cite{karimireddy_error_2019,stich_error-feedback_2020,richtarik_ef21_2021}.
The vast literature on the subject is further complicated by the many variants of these methods that have been proposed, see e.g.,~\citet{zheng_communication-efficient_2019,li2022analysis,tang2021errorcompensatedx,fatkhullin2023momentum,tian2026ef21,condat2022ef,fatkhullin2025ef21,gruntkowska2025error,egger2025bicompfl,redie2026sapefstepaheadpartialerror}, among others.

Many of the above papers provide \textit{theoretical guarantees}, most often in
the form of convergence rates that upper-bound the value of a hand-crafted
metric. Yet these bounds can be loose, the analysis pessimistic, or the
metric itself poorly matched to the phenomenon of interest, which complicates
meaningful comparisons between algorithms. Building on recent advances in
computer-aided proofs for optimization and automated construction of Lyapunov functions,
\citet{bergthomsen2025errorfeedback} derived \textit{tight}
convergence rates for $\CGD$, $\EF$, and $\EFtw$. Their rates are optimistic compared to existing
results and give a \textit{definitive} characterization of the worst-case
convergence of these methods on smooth and strongly convex functions; notably,
they also show that $\EF$ and $\EFtw$ attain the same optimized worst-case
rate. However, this analysis is restricted to the single-agent setting $n=1$,
which remains a significant theoretical and practical limitation.

\subsection{Contributions.}
  In this paper, we provide \textit{tight analyses} of both \(\EF\) and \(\EFtw\) in the multi-agent setting, \(n>1\),
under contractive compression (\Cref{as:compression}) and strong convexity and smoothness of the individual objectives.
The formal guarantees cover the regimes stated in the detailed contributions below. Outside those regimes,
our general-heterogeneity claims are stated as empirically validated laws. We show that the behavior in the distributed setting can differ from the single agent case, in particular due to \textit{heterogeneity} between agents. 
We distinguish between two types of heterogeneity: statistical heterogeneity (local minimizers
can differ across agents) and heterogeneous regularity parameters (local smoothness and
strong-convexity constants can differ across agents).

These rates are \textit{tight} in the method-specific sense used throughout
the paper: within the stated state and candidate Lyapunov classes, they identify the
\textit{best tuning}, the \textit{best Lyapunov certificate}, and the
\textit{smallest worst-case contraction factor}. They should not be read as
method-agnostic communication-complexity lower bounds over all compressed distributed
methods. To obtain them, we combine advanced proof techniques with numerical evaluations
that corroborate the theoretical predictions. Altogether, our results provide a
systematic picture of the worst-case behavior of these methods in the regimes
covered by our formal guarantees, complemented by empirical laws for the general
heterogeneous case.
  
\paragraph{Rigor: theorems, certificates, and empirical laws.} 
Most of our results are stated either as a theorem with a complete proof or as
a counterexample. To facilitate verification, we also attach
\textbf{proof certificates} to each formal statement, which serve as either
analytical or numerical validation of the stated guarantee. Specifically,
certificates of correctness are provided either
\begin{enumerate}
[label=(\roman*), topsep=0pt, itemsep=0pt, leftmargin=*]
  \item
    by a
    \textit{Computer Algebra System} (CAS) via a Wolfram Language script, or
  \item by numerically solving the associated \textit{Performance Estimation Problems} (PEP).
\end{enumerate}
CAS certificates
verify algebraic identities, whereas PEP certificates provide numerical confirmation of complete
statements. In the paper, the \includegraphics[height=1em]{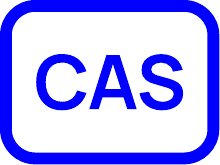}
and \includegraphics[height=1em]{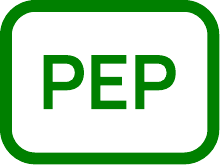} certificate badges indicate the available certificates
and link directly to the corresponding Wolfram
Language script or Jupyter notebook in the public GitHub repository.\footnotemark
\footnotetext{These certificates do not replace the mathematical proofs in the paper; rather, they provide an
  additional layer of transparency and error checking, analogous to unit tests in software development.
  They offer a reproducible, independently verifiable record supporting the theoretical claims and help
reduce the risk of oversights in complex derivations.}

For cases in which a formal proof could not be identified, we state certain results 
as \textbf{empirical laws} instead. These laws are supported by extensive numerical evidence
indicating that a proposed closed-form expression matches, up to numerical precision,
the corresponding analytical quantity, which may itself lack a closed-form representation.

\begin{table*}[t]
  \centering
  \vspace{0pt}
  \centering
  \small
  \begin{tblr}{
    colspec={|Q[c,m,wd=0.07\linewidth]|Q[c,m,wd=0.16\linewidth]|Q[c,m,wd=0.23\linewidth]|Q[c,m,wd=0.20\linewidth]|Q[c,m,wd=0.21\linewidth]|},
    colsep=5pt,
    rowsep=3pt,
    row{1}={font=\bfseries},
    hline{1,2,7}={1-5}{},
    hline{3,6}={1-4}{},
    hline{4}={2,3}{},
    hline{5}={2-4}{},
  }
    Agents & Heterogeneity & \(\boldsymbol{\EF}\) & \(\boldsymbol{\EFtw}\) & \(\boldsymbol{\EC}\) \\
    $n=1$ & \textemdash & \SetCell[c=2]{c,m}{\citet{bergthomsen2025errorfeedback}} & & \SetCell[r=5]{c,m}{Empirical Law~\ref{law:econtrol_tuning} \href[pdfnewwindow=true]{https://github.com/DanielBergThomsen/distributed-error-feedback/blob/main/certificates/empirical_laws/empirical_law_econtrol_tuning.ipynb}{\PEPbadge}} \\
    \SetCell[r=3]{c,m}{$n>1$} & None & Theorem~\ref{thm:ef} \href[pdfnewwindow=true]{https://github.com/DanielBergThomsen/distributed-error-feedback/blob/main/certificates/EF/theorem2.ipynb}{\PEPbadge}\,\href[pdfnewwindow=true]{https://github.com/DanielBergThomsen/distributed-error-feedback/blob/main/certificates/EF/ef_multiworker.wls}{\CASbadge} & \SetCell[r=2]{c,m}{Theorem~\ref{thm:ef21} \href[pdfnewwindow=true]{https://github.com/DanielBergThomsen/distributed-error-feedback/blob/main/certificates/EF21/theorem1.ipynb}{\PEPbadge}\,\href[pdfnewwindow=true]{https://github.com/DanielBergThomsen/distributed-error-feedback/blob/main/certificates/EF21/ef21_multiworker.wls}{\CASbadge}} & \\
    & Statistical\textsuperscript{\(\dagger\)} & \SetCell{c,m}{\textit{Counterexample}\\\footnotesize 2-cycle (Proposition~\ref{prop:counterexample_ef})} & & \\
    & Regularity parameters\textsuperscript{\(\ddagger\)} & \SetCell[c=2]{c,m}{Corollary~\ref{cor:ef21_linear} (Linear)\textsuperscript{*}} & & \\
    $n=2$ & Both\textsuperscript{\(\dagger\)}\textsuperscript{\(\ddagger\)} & \SetCell[c=2]{c,m}{Empirical Law~\ref{law:rate_n2} \href[pdfnewwindow=true]{https://github.com/DanielBergThomsen/distributed-error-feedback/blob/main/certificates/empirical_laws/empirical_law_n2_rate.ipynb}{\PEPbadge}} & & \\
\end{tblr}
  \caption{Summary of convergence results and empirical tuning laws for \(\EFtw\), \(\EF\), and \(\EC\). \(\dagger\) Statistical heterogeneity means that the local minimizers \(x_\star^{(i)}\) are not identical; \(\ddagger\) heterogeneous regularity parameters mean that the local smoothness and strong-convexity constants are not identical; \(*\) the linear corollary additionally assumes (beyond contractive compression) that \(\C\) is deterministic, additive, and positively homogeneous (\(\C(x+y)=\C(x)+\C(y)\), \(\C(\alpha x)=\alpha \C(x)\) for \(\alpha\ge0\)). Heterogeneous guarantees for \(\EFtw\) and \(\EF\) are currently limited to linear compressors and the \(n=2\) empirical law. PEP and CAS badges link to certificates where available.}
  \label{tab:intro_summary}

\end{table*}

\paragraph{Detailed contributions and outline.} More precisely, we make the following contributions, summarized in \Cref{tab:intro_summary}.
\begin{enumerate}[label=(\roman*), topsep=0pt, itemsep=0pt, leftmargin=*]
  \item In \Cref{sec:background}, we exhibit simple quadratic counterexamples showing that, for contractive compressors, both $\CGD$ and
    classic \(\EF\) exhibit cycles in heterogeneous multi-agent settings and therefore
    do not converge in general. 
  \item For \(\EFtw\) under contractive compression and statistical heterogeneity, we derive in \Cref{subsec:eftw} a tight Lyapunov
    function, optimal step size, and worst-case contraction factor, and the optimal rate is
    shown to be independent of the number of agents.
  \item For deterministic, additive, positively homogeneous compressors with heterogeneous regularity parameters, we also prove a sharp linear convergence guarantee for both \(\EFtw\) and \(\EF\) with averaged parameters \((\bar L, \bar \mu)\).
  \item In \Cref{subsec:efhomo}, we construct a tight Lyapunov function for classic \(\EF\) under statistical homogeneity, and show
    that its optimal step size and rate coincide with those of \(\EFtw\), recovering the single-agent results for any \(n\).
  \item Finally, in \Cref{sec:empirical_laws}, we state \textit{empirical laws} that characterize 
  the optimal step size for \(\EF\) 
        and \(\EFtw\) under general heterogeneity, an optimal Lyapunov function for \(\EFtw\),
	    an explicit polynomial rate formula for the \(n=2\) case, and an optimal two-parameter 
 	    tuning rule for \(\EC\)~\citep{gao2023econtrol}, supported by extensive numerical 
	    verification via numerical solution of the corresponding performance estimation problems.
\end{enumerate}

\begin{figure*}[t]
  \centering
  \begin{minipage}[t]{0.49\textwidth}
    \begin{algorithm}[H]
      \caption{Classic error feedback --- $\EF$}
      \label{alg:ef}
      \begin{algorithmic}[1]
        \STATE{\bfseries initialization:} \(x_0 \in \R^d, \eta > 0, e^{(i)}_0 = 0\) for each \(i \in [n]\);
        seeds \((\omega_k^{(i)})_{k \ge 0}\) given
        \FOR{\(k = 0, 1, 2, \ldots, K-1\)}
        \STATE Each agent \(i \in [n]\) compresses \(e_k^{(i)} + \eta \nabla f^{(i)}(x_k)\) and communicates \(m_k^{(i)} \coloneqq \C(e_k^{(i)} + \eta \nabla f^{(i)}(x_k); \omega_k^{(i)})\)
        \STATE Each agent \(i \in [n]\) updates \(e_{k+1}^{(i)} \gets e_k^{(i)} + \eta \nabla f^{(i)}(x_k) - m_k^{(i)}\)
        \STATE Server updates \(x_{k+1} \gets x_k - \frac{1}{n}\sum_{i=1}^n m_k^{(i)}\)
        \ENDFOR
      \end{algorithmic}
    \end{algorithm}
  \end{minipage}
  \hfill
  \begin{minipage}[t]{0.49\textwidth}
    \begin{algorithm}[H]
      \caption{Error Feedback 21 --- $\EFtw$}
      \label{alg:ef21}
        \begin{algorithmic}[1]
          \STATE{\bfseries initialization:} \(x_0 \in \R^d;\)
          step size \(\eta > 0;\)
        seeds \((\omega_k^{(i)})_{k \ge -1}\) given;
        \(d^{(i)}_0 = \C(\nabla f^{(i)}(x_0); \omega_{-1}^{(i)})\) for each \(i \in [n]\);
        \FOR{\(k = 0, 1, 2, \ldots, K-1\)}
        \STATE Server updates \(x_{k+1} \gets x_k - \eta \cdot \frac{1}{n} \sum_{i=1}^n d^{(i)}_k\)
        \STATE Each agent \(i \in [n]\) compresses \(\nabla f^{(i)}(x_{k+1}) - d^{(i)}_k\) and communicates \(m^{(i)}_k := \C(\nabla f^{(i)}(x_{k+1}) - d^{(i)}_k; \omega_k^{(i)})\)
        \STATE Each agent \(i \in [n]\) updates \(d^{(i)}_{k+1} \gets d^{(i)}_k + m^{(i)}_k\)
        \ENDFOR
      \end{algorithmic}
    \end{algorithm}
  \end{minipage}
\end{figure*}

In the next section, we introduce the algorithms, assumptions and counterexamples.

\section{Background}\label{sec:background}
This section motivates our multi-agent analysis by showing that classic $\EF$ can fail under
heterogeneity unless additional structure is imposed.
Specifically, \Cref{sec:motivation} provides counterexamples, and \Cref{sec:related_work} reviews
existing theoretical results on $\EF$ and $\EFtw$.
The definitions and notation needed to state the main results of this paper are provided in 
\Cref{sec:assumptions,sec:methodology}.

\subsection{Non-Convergence under Statistical Heterogeneity}\label{sec:motivation}
The following counterexamples demonstrate that neither $\CGD$ nor $\EF$ can
achieve arbitrary accuracy in the presence of statistical heterogeneity.

\paragraph{Compressed Gradient Descent.}
Consider the case where there are $n=2$ agents, each having access to first-order
oracles querying the one-dimensional quadratic functions
\begin{equation}\label{eq:counterexample_functions}
  f^{(1)}(x) \defeq \frac{\mu}{2} x^2 - x, \qquad f^{(2)}(x) \defeq \frac{\mu}{2} x^2 + x,
\end{equation}
where $\mu > 0$ is a constant. These functions are $\mu$-strongly convex
and $L$-smooth for any $L > \mu$. Set
\begin{equation}\label{eq:x0_cgd}
  x_0 \defeq \eta \frac{\sqrt{\epsilon}}{2 - \eta \mu}.
\end{equation}
By definition, for $\CGD$
\begin{equation}\label{eq:x1_cgd}
  x_1 = x_0 - \frac{\eta}{2} \left[\C((f^{(1)})'(x_0); \omega_0^{(1)}) + \C((f^{(2)})'(x_0); \omega_0^{(2)}) \right].
\end{equation}
Under \Cref{as:compression}, the compression oracle may respond
\begin{equation}
  \begin{aligned}
    \C((f^{(1)})'(x_0); \omega_0^{(1)}) &= (1 - \sqrt{\epsilon}) (f^{(1)})'(x_0), \\
    \C((f^{(2)})'(x_0); \omega_0^{(2)}) &= (1 + \sqrt{\epsilon}) (f^{(2)})'(x_0).
  \end{aligned}
\end{equation}
Plugging this, the derivatives of $f^{(1)}$ and $f^{(2)}$, and the definition of $x_0$ into
\eqref{eq:x1_cgd},
\begin{equation}
  \begin{aligned}
    x_1 &= x_0 - \frac{\eta}{2} \left[2\mu x_0 + 2\sqrt{\epsilon} \right] = x_0(1 - \eta \mu) - \eta \sqrt{\epsilon} \\
    &= \frac{\eta \sqrt{\epsilon}}{2 - \eta \mu} \left[(1 - \eta \mu) - (2 - \eta \mu)\right] = -x_0.
  \end{aligned}
\end{equation}
By symmetry, if the compression oracle next responds with
\begin{equation}
  \begin{aligned}
    \C((f^{(1)})'(x_1); \omega_1^{(1)}) &= (1 + \sqrt{\epsilon}) (f^{(1)})'(x_1), \\
    \C((f^{(2)})'(x_1); \omega_1^{(2)}) &= (1 - \sqrt{\epsilon}) (f^{(2)})'(x_1).
  \end{aligned}
\end{equation}
then the same computation gives $x_2 = x_0$, yielding a 2-step cycle.

\paragraph{Error Feedback.}
Consider now exactly the same functions defined in \eqref{eq:counterexample_functions}. The following
proposition shows that the same behavior can be observed in $\EF$.
\begin{restatable}{proposition}{propositionCounterexampleEF}\label{prop:counterexample_ef}
  Let there be $n=2$ agents, each having access to first-order oracles
  querying the one-dimensional quadratic functions defined in \eqref{eq:counterexample_functions}.
  Then, for each fixed step size \(\eta>0\) in the following cases, there exist
  admissible compressor responses under \Cref{as:compression} for which the full \(\EF\)
  state \((x_k,e_k^{(1)},e_k^{(2)})\) is $2$-periodic:
  \begin{enumerate}[topsep=0pt, itemsep=0pt, leftmargin=*]
    \item $\eta < \frac{2}{\mu}$, with $x_0 = \eta \frac{\sqrt{\epsilon}}{2 - \eta \mu}$.
    \item $\eta > \frac{2}{\mu}$, with $x_0 = -\eta \frac{\sqrt{\epsilon}}{2 - \eta \mu}$.
    \item $\eta = \frac{2}{\mu}$, with any $x_0$.
  \end{enumerate}
\end{restatable}
A simple proof of this is provided in \Cref{sec:proof_counterexample_ef}.

\subsection{Related Work}\label{sec:related_work}
Error feedback has a long history in signal processing, where it is used to compensate for
quantization in communication~\cite{cutler1952differential,inose2005unity}.
In distributed optimization,
it dates back to the classical work of \citet{seide_1-bit_2014}. Following its introduction, subsequent
work has studied the behavior of error feedback under various assumptions, including analyses for
specific compression operators such as deterministic Top-$K$ \cite{alistarh_convergence_2018} and
stochastic compressors \cite{wu_error_2018}.

A complementary line of work leverages the contractive properties of compression operators to derive general
convergence guarantees for a broad class of operators. In this setting, it is common to distinguish between
\emph{biased} and \emph{unbiased} compressors \cite{beznosikov_biased_2020}. Convergence rates for error feedback
with contractive compressors have been established for strongly convex functions
\cite{stich_sparsified_2018}, quasi-convex and nonconvex functions \cite{karimireddy_error_2019},
and using stochastic gradients \cite{stich_error-feedback_2020}. Tight rates have also been
established under the same assumptions as this work, for the single-agent regime \cite{bergthomsen2025errorfeedback}.

Error Feedback 21 ($\EFtw$) \cite{richtarik_ef21_2021} is a variant of
error feedback that was designed specifically to handle the heterogeneity of the multi-agent
setting. The same work also reports experiments on logistic regression with a nonconvex regularizer.
$\EFtw$ and its variants have been studied in
many different scenarios, some of which include using stochastic gradients, momentum \cite{fatkhullin2023momentum}
and practical extensions such as bidirectional compression, variance reduction,
and proximal setups \cite{fatkhullin_ef21_2021}.

A related two-parameter method is \(\EC\)~\citep{gao2023econtrol}, which introduces a controllable
error-compensation mechanism combining ideas from \(\EF\) and \(\EFtw\). It is shown to converge
in strongly convex, convex, and nonconvex settings.

\subsection{Assumptions and Notations}\label{sec:assumptions}
The following definition is used throughout this work.
\begin{definition}[Class \(\Fml\)]\label{def:Fml}
  For constants \(0<\mu<L\), denote by \(\Fml\) the set of functions \(h: \mathbb{R}^d \to \mathbb{R}\) that are \(L\)-smooth and
  \(\mu\)-strongly convex. That is, for any $h \in \Fml$, and any \(x, y \in \mathbb{R}^d\), it holds that
  \[
    h(y) \leq h(x) + \langle \nabla h(x), y - x \rangle + \frac{L}{2} \|y - x\|^2,
  \]
  and
  \[
    h(y) \geq h(x) + \langle \nabla h(x), y - x\rangle + \frac{\mu}{2} \|y - x\|^2.
  \]
\end{definition}
Throughout this work, each local function \(f^{(i)}\) is assumed to belong to \(\mathcal{F}_{\mu^{(i)}, L^{(i)}}\).

The symbol \(\SO^\ell\) denotes the set of symmetric matrices, and \(\SO^\ell_+\) denotes the set of positive semidefinite matrices. For any two matrices \(A\in \SO^\ell\) and \(B\in \SO^d\), the Kronecker product is denoted by \(A\otimes B\).
The condition number is denoted by \(\kappa \coloneqq \frac{L}{\mu}\). For any objective function \(f\in \Fml\),
the minimizer is denoted by \(x_\star\coloneqq \arg\min_{x \in \mathbb{R}^d} f\), and its minimum value is denoted by
\(f_\star \coloneqq \min_{x \in \mathbb{R}^d} f(x)\). For each local function \(f^{(i)}\), its unique minimizer is denoted
by \(x_\star^{(i)}\) and its minimum value by \(f_\star^{(i)}\).

\subsection{Methodology}\label{sec:methodology}
The analysis contained in \Cref{sec:main_results} relies on the systematic identification of Lyapunov functions
that provide a tight convergence rate for each method. These Lyapunov functions are \emph{optimal} with respect
to a large class of Lyapunov functions, defined in this section.

\paragraph{Lyapunov functions.}
Let \(\M: (\R^d)^\ell \times \R^d \times \mathcal F \to (\R^d)^\ell \times \R^d \) denote a first-order
method acting on a set of functions \(\mathcal F\) of dimension \(d\), for an integer \(\ell \in \mathbb N\) different \emph{state variables}.
Such a method, given a function \(f\in \mathcal F\), is applied to an initial \textit{state} \(\xi_0 \in (\R^d)^\ell\)
and iterate \(x_0 \in \R^d\), and generates the next state \(\xi_1\) and next iterate \(x_1\). The
\textit{states} represent information summarizing the current
point in the optimization trajectory, including auxiliary quantities that the algorithms may depend on---for
example, error-related quantities in error feedback algorithms.

\begin{definition}[Candidate Lyapunov function]\label{def:candidate_lyapunov}
  A function \(\V: (\mathbb{R}^d)^\ell \times  \mathbb{R}^d \times \mathcal F \to \mathbb{R}\), written \(\V(\xi,x;f)\), is called a candidate Lyapunov function if, for every \(f\in\mathcal F\), it satisfies the following conditions:
  \begin{enumerate}[leftmargin=*, itemsep=0pt]
    \item (Non-negativity) \(\V(\xi, x; f) \geq 0\), for any \(\xi \in (\mathbb{R}^d)^\ell\), \(x \in \mathbb{R}^d\),
    \item (Zero at fixed-point) \(\V(\xi, x; f) = 0\) if and only if {\(x = x_\star\)} and \(\xi = \xi_\star\) for a unique \(\xi_\star \in (\mathbb{R}^d)^\ell\).
    \item (Meaningful lower bound) there exists a positive semidefinite matrix \(P \in \SO^\ell_+\) and a scalar \(p \geq 0\) such that,
      for all \(\xi \in (\mathbb{R}^d)^\ell\) and \(x \in \mathbb{R}^d\),
      \[
        \V(\xi, x; f) \geq (\xi - \xi_\star)^\top (P \otimes I_d) (\xi - \xi_\star) + p (f(x) - f_\star),
      \]
      and \(\tr(P) + p = 1\).
  \end{enumerate}
  In quadratic forms involving \(P \otimes I_d\), we identify
  \(\xi=(\xi^{(1)},\ldots,\xi^{(\ell)})\in(\mathbb R^d)^\ell\) with the stacked vector
  \([(\xi^{(1)})^\top,\ldots,(\xi^{(\ell)})^\top]^\top\in\mathbb R^{\ell d}\), where each block
  \(\xi^{(j)}\in\mathbb R^d\) is one state variable.
\end{definition}

The objective is to find candidate Lyapunov functions \(\V\) satisfying the recurrence
\begin{equation}\label{eq:lyapunov_recurrence}
  \V(\xi_{1}, x_{1}; f) \leq \rho \cdot \V(\xi_{0}, x_{0}; f),
\end{equation}
for some constant \(\rho < 1\), uniformly over \(\mathcal F\). Finding the \emph{optimal}
Lyapunov function within a parameterized class for a method \(\M\) amounts to solving the following problem:
\begin{equation}\label{eq:lyapunov_problem}
  \begin{aligned}
    \rhostar{\M} \coloneq \min_{\mathcal{V}} \quad & \sup_{\substack{f \in \Fml \\ \xi_0, x_0 \\ \V(\xi_0,x_0;f)>0}} \frac{\V(\xi_1, x_1; f)}{\V(\xi_0, x_0; f)} \\
    & \text{s.t. } (\xi_{1}, x_{1}) = \Mfxxi{\xi_0}{x_0}{f}.
  \end{aligned}
\end{equation}
The goal of this work is to identify optimal Lyapunov functions for \(\EF\) and \(\EFtw\) and formally prove that they
achieve the convergence rate defined in \eqref{eq:lyapunov_problem}.
It has been shown that optimal candidate Lyapunov functions can be identified by solving
semidefinite programs (SDPs), yielding numerical convergence guarantees
\citep{taylor2018lyapunov}.
Control-inspired SDP analyses have also been used for decentralized optimization over
communication graphs~\citep{sundararajan2019canonical,sundararajan2020analysis}.

We used numerical tools both to obtain guarantees~\citep{taylor2017performance,goujaud2022pepit} and to search for
Lyapunov functions~\citep{taylor2018lyapunov,upadhyaya2023automated}. Such numerical evidence is not, on its own, a theoretical
proof, but it can reveal stable certificate structures that can then be simplified and proved analytically. In the
multi-agent setting, the optimal Lyapunov coefficients need not be unique, so Appendix~\ref{sec:additional_experiments}
describes the SDP-guided Lyapunov-search heuristics used to infer closed-form structures.

\begin{algorithm}[H]
  \caption{Error Control --- \(\EC\)}
  \label{alg:econtrol}
  \begin{algorithmic}[1]
    \STATE{\bfseries initialization:} \(x_0 \in \R^d;\)
    parameters \(\eta \ge 0\) and \(\gamma > 0;\)
    \(e^{(i)}_0 = 0, d^{(i)}_0 = 0\) for \(i \in [n];\)
    seeds \((\omega_k^{(i)})_{k \ge 0}\) given
    \FOR{\(k = 0, 1, 2, \ldots, K-1\)}
    \STATE Each agent \(i \in [n]\) communicates the compressed message
    \(m^{(i)}_k := \C(\eta e_k^{(i)} + \nabla f^{(i)}(x_k) - d_k^{(i)}; \omega_k^{(i)})\)
    \STATE Each agent \(i \in [n]\) updates \(d^{(i)}_{k+1} \gets d^{(i)}_k + m^{(i)}_k\)
    \STATE Server updates \(x_{k+1} \gets x_k - \gamma \cdot \frac{1}{n} \sum_{i=1}^n d^{(i)}_{k+1}\)
    \STATE Each agent \(i \in [n]\) updates \(e^{(i)}_{k+1} \gets e^{(i)}_k + \nabla f^{(i)}(x_k) - d^{(i)}_{k+1}\)
    \ENDFOR
  \end{algorithmic}
\end{algorithm}
\section{Main Results} \label{sec:main_results}
This section states tight convergence guarantees for \(\EFtw\) under contractive compression, a linear-compressor
corollary with heterogeneous regularity parameters, and corresponding guarantees for \(\EF\) under statistical homogeneity.
In light of the counterexample in \Cref{prop:counterexample_ef}, the result for \(\EF\) requires the additional
assumption that the local objectives \(f^{(i)}\) share the same minimizer.

\subsection{Error Feedback 21}\label{subsec:eftw}
For $\EFtw$, the state is defined as
\begin{equation}\label{eq:ef21_state}
  \xi_k^{\EFtw} \coloneqq
  \begin{bmatrix}
    [x_k, \dots, x_k]^\top \\
    [\nabla f^{(1)}(x_k), \dots, \nabla f^{(n)}(x_k)]^\top \\
    [d_k^{(1)}, \dots, d_k^{(n)}]^\top
  \end{bmatrix}.
\end{equation}
Its fixed-point value is
\[
  \xi_\star^{\EFtw} \coloneqq
  \begin{bmatrix}
    [x_\star, \dots, x_\star]^\top \\
    [\nabla f^{(1)}(x_\star), \dots, \nabla f^{(n)}(x_\star)]^\top \\
    [\nabla f^{(1)}(x_\star), \dots, \nabla f^{(n)}(x_\star)]^\top
  \end{bmatrix}.
\]
Candidate Lyapunov functions may thus include many terms, but the Lyapunov function
given in \Cref{thm:ef21} is relatively simple, and is worst-case optimal under the optimal step
size tuning.
\begin{certified}[cas={https://github.com/DanielBergThomsen/distributed-error-feedback/blob/main/certificates/EF21/ef21_multiworker.wls}, pep={https://github.com/DanielBergThomsen/distributed-error-feedback/blob/main/certificates/EF21/theorem1.ipynb}]{theorem}{theoremEFtw} \label{thm:ef21}
  Let \(\epsilon \in (0, 1)\) and assume that the compression operator $\mathcal C$ satisfies Assumption~\ref{as:compression}.
  Let $f^{(i)} \in \Fml$ for each $i \in [n]$.
  Let the step size be given by
  \begin{equation}\label{eq:opt_step}
    \eta^\star = \left( \frac{2}{L + \mu} \right) \cdot \left( \frac{1 - \sqrt{\epsilon}}{1 + \sqrt{\epsilon}} \right).
  \end{equation}
  Then, the Lyapunov function
  \begin{equation} \label{eq:ef21_lyap}
    \begin{aligned}
      \mathcal V\!\left(\xi^{\EFtw},x;f\right)
      \coloneq & \frac{\sqrt{\epsilon}}{n} \left\|\sum_{i=1}^n \nabla f^{(i)}(x_k)\right\|^2 \\
      & + \sum_{i=1}^n \|\nabla f^{(i)}(x_k) - d^{(i)}_k\|^2
    \end{aligned}
  \end{equation}
  is optimal (i.e., solves \eqref{eq:lyapunov_problem}) and satisfies
  \begin{equation}
    \E_{\omega}\!\left[\mathcal V\!\left(\xi_{k+1}^{\EFtw},x_{k+1};f\right)\right]
    \le \rho_\star \cdot \E_{\omega}\!\left[\mathcal V\!\left(\xi_{k}^{\EFtw},x_{k};f\right)\right]
  \end{equation}
  where the rate is given by
  \begin{equation}
    \rho_\star \coloneq\
    \sqrt{\epsilon} + \left(\frac{1 - \sqrt{\epsilon}}{2}\right) \left(\frac{\kappa - 1}{\kappa + 1}\right)^2 \Psi(\kappa, \epsilon)
    \label{eq:opt_rate}
  \end{equation}
  and
  \begin{equation}
    \Psi(\kappa, \epsilon) \coloneq 1 - \sqrt{\epsilon} + \sqrt{(1 + \sqrt{\epsilon})^2 + \sqrt{\epsilon}\, 16 \frac{\kappa}{(\kappa - 1)^2}}.
  \end{equation}
  Finally, the step size in~\eqref{eq:opt_step} is worst-case optimal for $\EFtw$: within the candidate Lyapunov class based on \(\xi^\EFtw\),
  it achieves the minimal worst-case one-step contraction factor, \(\rho_\star\) in \eqref{eq:opt_rate}. This bound is also multi-step tight,
  meaning that after $k$ iterations the worst-case contraction equals \(\rho_\star^k\).
\end{certified}

A proof appears in \Cref{sec:proof_ef21}. The numerical and symbolic certificates can be accessed by clicking the PEP and CAS badges in the theorem header.

The symmetry across agents implies that the worst-case rate in \Cref{thm:ef21} matches the single-agent
rate: the worst-case instance for \(n=1\) can be replicated across agents. Though the single-agent Lyapunov function
is recovered when setting \(n=1\), the multi-agent Lyapunov function is not a trivial extension of the single-agent
case. One could have expected a weighted sum of the single-agent Lyapunov functions defined in \cite{bergthomsen2025errorfeedback}
to be sufficient, but it is clear that the weighting placed on the individual terms of \eqref{eq:ef21_lyap} do not correspond
to that.

The next corollary extends \Cref{thm:ef21} to heterogeneous regularity parameters (agent-specific smoothness and strong-convexity constants)
under deterministic, additive, positively homogeneous compression, a setup in which \citet{richtarik_ef21_2021} showed an equivalence between $\EF$ and $\EFtw$, under a reparametrization.
\begin{certified}{corollary}{corollaryEFtwLinear} \label{cor:ef21_linear}
  Let \(\epsilon \in [0, 1)\) and assume that the compression operator $\mathcal C$ satisfies
  Assumption~\ref{as:compression} and is deterministic, additive, and positively homogeneous (so
  \(\C(x+y) = \C(x) + \C(y)\) and \(\C(\alpha x) = \alpha \C(x)\) for all \(\alpha \ge 0\)).
  Let \(f^{(i)} \in \mathcal{F}_{\mu^{(i)}, L^{(i)}}\) for each \(i \in [n]\), and define
  \[
    \bar L \coloneqq \frac{1}{n} \sum_{i=1}^n L^{(i)}, \qquad
    \bar \mu \coloneqq \frac{1}{n} \sum_{i=1}^n \mu^{(i)}.
  \]
  \[
    \kappa_\Sigma \coloneqq \frac{\bar L}{\bar \mu}
    = \frac{\sum_{i=1}^n L^{(i)}}{\sum_{i=1}^n \mu^{(i)}}.
  \]
  Let the step size be given by
  \[
    \eta^\star = \left( \frac{2}{\bar L + \bar \mu} \right)
    \cdot \left( \frac{1 - \sqrt{\epsilon}}{1 + \sqrt{\epsilon}} \right).
  \]
  Then the Lyapunov function
  \begin{equation}\label{eq:ef21_linear_lyap}
    \mathcal V_{\mathrm{lin}}\!\left(\xi^{\EFtw},x;f\right)
    \coloneq \sqrt{\epsilon}\,\|\bar g_k\|^2 + \|\bar g_k - \bar d_k\|^2
  \end{equation}
  satisfies the deterministic contraction
  \[
    \mathcal V_{\mathrm{lin}}\!\left(\xi_{k+1}^{\EFtw},x_{k+1};f\right)
    \le \rho_\star \cdot
    \mathcal V_{\mathrm{lin}}\!\left(\xi_{k}^{\EFtw},x_{k};f\right),
  \]
  where \(\bar g_k \coloneqq \frac{1}{n}\sum_{i=1}^n \nabla f^{(i)}(x_k)\),
  \(\bar d_k \coloneqq \frac{1}{n}\sum_{i=1}^n d_k^{(i)}\), and the rate \(\rho_\star\) is given by
  \eqref{eq:opt_rate} with \(\kappa = \kappa_\Sigma\).
\end{certified}

\begin{remark}
  Under the assumptions of \Cref{cor:ef21_linear}, distributed \(\CGD\) also converges, since the
  method reduces to a single-agent compressed iteration on the averaged objective.
  The single-agent result of \citet{bergthomsen2025errorfeedback} then shows that \(\CGD\) outperforms both \(\EF\) and \(\EFtw\) in this regime.
\end{remark}

A proof appears in \Cref{sec:proof_ef21_linear}.
The same step size and rate apply to $\EF$, matching Theorems~1--2 in
\citet{bergthomsen2025errorfeedback} evaluated at \(\kappa_\Sigma\).
	The result follows by observing that the barred variables evolve exactly as a single-agent $\EFtw$
	instance with parameters \((\bar L, \bar \mu)\), and then invoking the $\EF$--$\EFtw$ equivalence.
	Extending tight guarantees beyond this deterministic linear setting appears technically challenging, motivating the next section.

\subsection{Classic Error Feedback under Statistical Homogeneity} \label{subsec:efhomo}
Due to the counterexample given in \Cref{prop:counterexample_ef}, an additional assumption is required
to prove convergence for $\EF$. This is the same assumption as that required for the linear
rate of convergence given in \cite{beznosikov_biased_2020}---namely, that the minimizers of the
local functions are all the same, a condition commonly known as the interpolation regime~\cite{ma2018power,vaswani2019fast}.
\begin{assumption}[Statistical homogeneity]\label{as:fixed_point_homo}
  The local objectives satisfy \emph{statistical homogeneity}, i.e., they share the same minimizer:
  \[
    \xs^{(i)} = \xs^{(j)} \quad \text{for all } i,j \in [n].
  \]
\end{assumption}

For $\EF$, the state is defined as
\begin{equation}
  \xi_k^{\EF} \coloneqq
  \begin{bmatrix}
    x_k \\
    [\nabla f^{(1)}(x_k), \dots, \nabla f^{(n)}(x_k)]^\top \\
    [m_k^{(1)}, \dots, m_k^{(n)}]^\top \\
    [e_k^{(1)}, \dots, e_k^{(n)}]^\top
  \end{bmatrix}.
\end{equation}
Its fixed-point value is
\[
  \xi_\star^{\EF} \coloneqq
  \begin{bmatrix}
    x_\star \\
    [0, \dots, 0]^\top \\
    [0, \dots, 0]^\top \\
    [0, \dots, 0]^\top
  \end{bmatrix},
\]
where \(\nabla f^{(i)}(x_\star)=0\) for all \(i\) under \Cref{as:fixed_point_homo}, so the communication and
error-feedback blocks are also zero at the fixed point.
Similar to \Cref{thm:ef21}, \Cref{thm:ef} provides a relatively simple Lyapunov function which is worst-case
optimal under the optimal step size tuning.

\begin{certified}[cas={https://github.com/DanielBergThomsen/distributed-error-feedback/blob/main/certificates/EF/ef_multiworker.wls}, pep={https://github.com/DanielBergThomsen/distributed-error-feedback/blob/main/certificates/EF/theorem2.ipynb}]{theorem}{theoremEF} \label{thm:ef}
  Let \(\epsilon \in (0, 1)\) and assume that the compression operator $\mathcal C$ satisfies Assumption~\ref{as:compression}.
  Let $f^{(i)} \in \Fml$ for each $i \in [n]$, and assume \Cref{as:fixed_point_homo}.

  Let the step size be given by $\eta^\star$ as defined in~\eqref{eq:opt_step}.
  Then, the Lyapunov function
  \begin{equation} \label{eq:ef_lyap}
    \begin{aligned}
      \mathcal V\!\left(\xi^{\EF},x;f\right)
      \coloneq & \frac{1}{n \sqrt{\epsilon}} \left\| \sum_{i=1}^n e^{(i)}_k \right\|^2  + \sum_{i=1}^n \|x_k - \xs - e_k^{(i)}\|^2.
    \end{aligned}
  \end{equation}
  is optimal and satisfies
  \begin{equation}\label{eq:rate_ef}
    \rhostar{\EF} = \rho_\star,
  \end{equation}
  where $\rho_\star$ is defined in~\eqref{eq:opt_rate}.
  Finally, the step size in~\eqref{eq:opt_step} is worst-case optimal for $\EF$: within the candidate Lyapunov class based on \(\xi^\EF\),
  it achieves the minimal worst-case one-step contraction factor, \(\rho_\star\) in \eqref{eq:opt_rate}. This bound is also multi-step tight,
  meaning that after $k$ iterations the worst-case contraction equals \(\rho_\star^k\).
\end{certified}

A proof appears in \Cref{sec:proof_ef}. Like in the case of $\EFtw$, the single-agent results from \cite{bergthomsen2025errorfeedback}
are recovered for any $n \geq 1$, including the optimal step size and the Lyapunov
function (by setting $n=1$ in \eqref{eq:ef_lyap}).

	Under shared regularity parameters \(f^{(i)} \in \Fml\), the worst-case contraction is
	independent of the number of agents \(n\) for \(\EFtw\) without requiring a shared minimizer,
	and for classic \(\EF\) under the shared-minimizer condition of \Cref{as:fixed_point_homo}. In both cases,
	the class-optimal step sizes and contraction factors coincide with their single-agent
	counterparts. In these distributed regimes, the corresponding single-agent analyses
	therefore characterize the worst-case rates.

\section{Empirical Laws with Heterogeneous Regularity Parameters}\label{sec:empirical_laws}
The results of the previous sections either assume homogeneous regularity parameters (shared parameters \((L, \mu)\))
or rely on linear deterministic compressors to handle heterogeneous regularity parameters. When this is relaxed to allow
general heterogeneous regularity parameters, with agent-specific parameters \(f^{(i)} \in \mathcal{F}_{\mu^{(i)}, L^{(i)}}\),
numerical evidence indicates that the worst-case behavior depends on the local smoothness and
strong-convexity parameters.
Based on extensive numerical experiments, this section states
empirical laws for (i) the optimal step size for both $\EF$ and $\EFtw$, (ii) a corresponding
Lyapunov function for $\EFtw$, (iii) the optimal convergence rate when $n=2$, and
(iv) a two-parameter tuning rule for $\EC$~\citep{gao2023econtrol} in the general case.

The empirical laws below are formulated based on numerical validation using the Performance Estimation
Problem (PEP) framework~\citep{drori2014contributions,taylor2017smooth}.
Appendix~\ref{sec:verification_details} gives the verification details, including exact grid searches and parameter sweeps.
The empirical formulas were tested across a wide range of problem parameters.
\begin{figure*}[t]
  \centering
  \includegraphics[width=.9\textwidth]{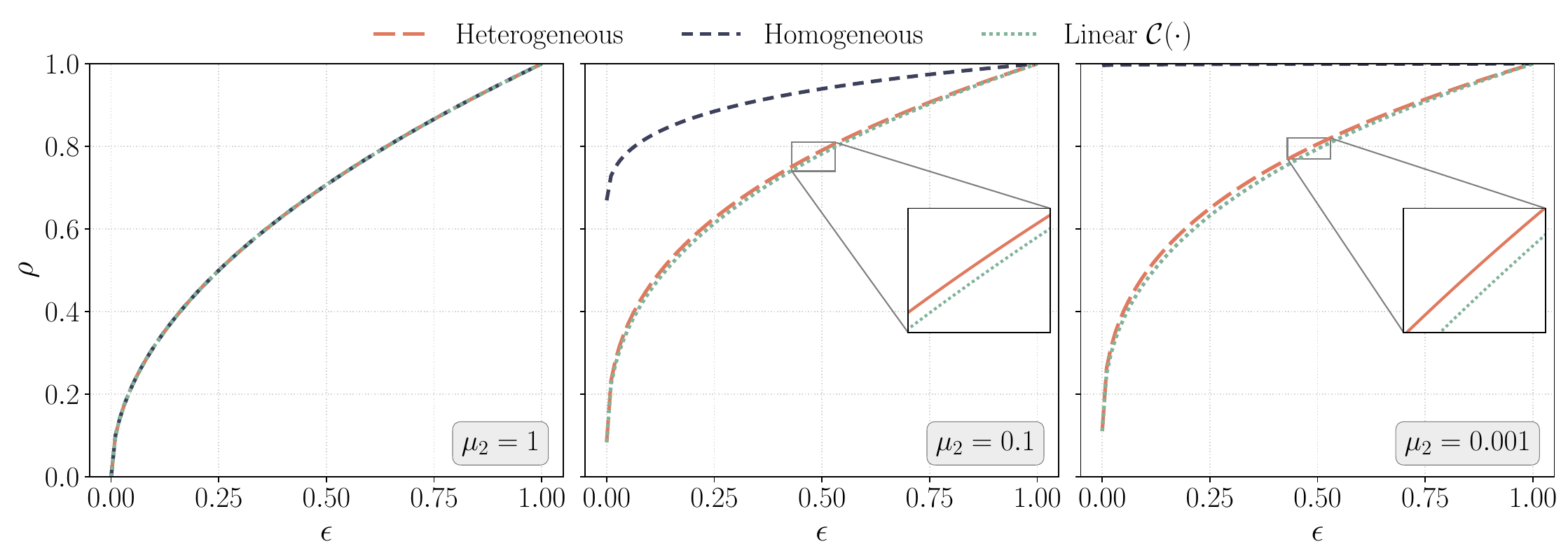}
  \caption{Comparison of the rates predicted by Empirical Law~\ref{law:rate_n2},
    Theorem~\ref{thm:ef21} with worst-case parameters \(\max_i L^{(i)}\) and
    \(\min_i \mu^{(i)}\), and Corollary~\ref{cor:ef21_linear} with averaged
    parameters. The inset in the rightmost panel magnifies the gap between the empirical-law rate
    and the linear-compressor rate. Here \(L^{(1)} = L^{(2)} = 1\), \(\mu^{(1)} = 1\), and
  \(\mu^{(2)} \in \{1, 0.1, 10^{-3}\}\) from left to right.}
  \label{fig:conj_vs_theorem_rates}
\end{figure*}

\begin{certified}[pep={https://github.com/DanielBergThomsen/distributed-error-feedback/blob/main/certificates/empirical_laws/empirical_law_step_size.ipynb}]
  {empiricallaw}{empiricalLawOptStep}[General optimal step size]\label{law:opt_step}
  Consider the setting with heterogeneous regularity parameters, i.e., let
  \(f^{(i)} \in \mathcal{F}_{\mu^{(i)}, L^{(i)}}\).
  The optimal step size for both $\EF$ and $\EFtw$ is given by
  \begin{equation}
    \eta^\star = \left( \frac{2n}{\sum_{i=1}^n (L^{(i)} + \mu^{(i)})} \right)
    \cdot \left( \frac{1 - \sqrt{\epsilon}}{1 + \sqrt{\epsilon}} \right).
  \end{equation}
\end{certified}
This step size reduces to~\eqref{eq:opt_step} when all agents share the same parameters.

Numerical evidence further suggests that the following Lyapunov function is optimal for $\EFtw$ among
Lyapunov functions built from the class defined by \eqref{eq:ef21_state}.
\begin{certified}[pep={https://github.com/DanielBergThomsen/distributed-error-feedback/blob/main/certificates/empirical_laws/empirical_law_lyapunov.ipynb}]
  {empiricallaw}{empiricalLawLyapEF21}[General Lyapunov function]\label{law:lyap_ef21}
  Let $S := \sum_{i=1}^n (L^{(i)} + \mu^{(i)})$ and
  $w_i := \frac{S}{L^{(i)} + \mu^{(i)}}$.
  Let each agent have heterogeneous regularity parameters, i.e.,
  \(f^{(i)} \in \mathcal{F}_{\mu^{(i)}, L^{(i)}}\). The optimal Lyapunov function for $\EFtw$ is
  given by
  \begin{equation}
    \begin{aligned}
      \mathcal V\!\left(\xi^{\EFtw},x;f\right) \coloneq
      & \frac{1}{n} \sum_{i=1}^n w_i \|\nabla f^{(i)}(x_k) - d^{(i)}_k\|^2 \\
      & + \frac{\sqrt{\epsilon}}{n} \left\|\sum_{i=1}^n \nabla f^{(i)}(x_k)\right\|^2
    \end{aligned}
  \end{equation}
\end{certified}
This family of Lyapunov functions reduces to~\eqref{eq:ef21_lyap} when all agents share the same parameters; in the heterogeneous case,
it weights each local estimator error by \(\sum_{j=1}^n (L^{(j)}+\mu^{(j)})/(L^{(i)}+\mu^{(i)})\).

Finally, numerical evidence suggests an explicit characterization of the
agent-sensitive convergence rate for both $\EF$ and $\EFtw$ when $n=2$:
\begin{certified}[pep={https://github.com/DanielBergThomsen/distributed-error-feedback/blob/main/certificates/empirical_laws/empirical_law_n2_rate.ipynb}]
  {empiricallaw}{empiricalLawRateN2}[General rate for $n=2$]\label{law:rate_n2}
  Let $n=2$ and let the agents have heterogeneous regularity parameters, i.e.,
  \(f^{(i)} \in \mathcal{F}_{\mu^{(i)}, L^{(i)}}\).
  At the step size \(\eta=\eta^\star\) from \Cref{law:opt_step}, the optimal
  convergence rate of both $\EF$ and $\EFtw$ is given by the largest real root of the polynomial
  \begin{equation}
    \label{eq:poly_n2}
    \begin{aligned}
      Q(\rho) ={}& \rho^3 - \left[s(2+s) + r(s)(s K_1 + K_2)\right] \rho^2 \\
      &+ s^2 \left[1+2s + r(s)(K_1 + s K_2)\right] \rho - s^4
    \end{aligned}
  \end{equation}
  where $s \coloneqq \sqrt{\epsilon}$, $r(s) \coloneqq \frac{(1-s)^2}{1+s}$, and the constants $K_1, K_2$ are defined as
  \begin{equation}
    K_1 \coloneqq \frac{\Delta_2^2 \Sigma_1 + \Delta_1^2 \Sigma_2}{\Sigma_1 \Sigma_2 (\Sigma_1 + \Sigma_2)},
    \qquad
    K_2 \coloneqq \frac{(\Delta_1 + \Delta_2)^2}{(\Sigma_1 + \Sigma_2)^2},
  \end{equation}
  with $\Sigma_i \coloneqq L^{(i)} + \mu^{(i)}$ and $\Delta_i \coloneqq L^{(i)} - \mu^{(i)}$ for $i \in \{1, 2\}$.
\end{certified}

The derivation of this polynomial and a visualization of its roots are provided in
Appendix~\ref{sec:additional_experiments}; see \Cref{fig:poly_bifurcation}.
\Cref{fig:conj_vs_theorem_rates} compares the empirical-law rate with the
homogeneous worst-case and linear-compressor rates. A further comparison with the
distributed $\EFtw$ rate of \citet{richtarik_ef21_2021} is provided in
Appendix~\ref{sec:additional_experiments}; see \Cref{fig:richtarik_comparison}.

Numerical evidence further suggests a simple tuning rule for the two-parameter $\EC$ method
~\citep{gao2023econtrol}, where
$\gamma$ denotes the model step size and $\eta$ denotes the coefficient on the error term.
\begin{certified}[pep={https://github.com/DanielBergThomsen/distributed-error-feedback/blob/main/certificates/empirical_laws/empirical_law_econtrol_tuning.ipynb}]
  {empiricallaw}{empiricalLawEControl}[\(\EC\) tuning]\label{law:econtrol_tuning}
  Consider the case of heterogeneous regularity parameters, i.e., \(f^{(i)} \in \mathcal{F}_{\mu^{(i)}, L^{(i)}}\).
  The empirically optimal tuning for $\EC$ satisfies
  \begin{equation}
    \begin{aligned}
      \eta^\star_{\EC} &= 0, \\
      \gamma^\star_{\EC} &= \left( \frac{2n}{\sum_{i=1}^n (L^{(i)} + \mu^{(i)})} \right)
      \cdot \left( \frac{1 - \sqrt{\epsilon}}{1 + \sqrt{\epsilon}} \right),
    \end{aligned}
  \end{equation}
  which coincides with the optimal step size in \Cref{law:opt_step}.
\end{certified}
When \(\eta = 0\) in \(\EC\), the method is equivalent to \(\EFtw\) under a reordering of the
corresponding parameter \(d_k^{(i)}\), and by instead initializing \(d_0^{(i)} = 0\). Consequently,
the algorithm inherits the same worst-case convergence rates as \(\EFtw\).

\section{Conclusion}
This paper provides tight worst-case analyses for distributed error-feedback algorithms in the
multi-agent setting.
Under homogeneous regularity parameters (the same smoothness and strong-convexity constants across agents), the optimal
step sizes and contraction factors for \(\EF\) and \(\EFtw\) are independent of the number of
workers \(n\) and coincide with the corresponding single-agent guarantees.

For \(\EFtw\) with contractive compression, the analysis allows statistical heterogeneity
and yields an optimal Lyapunov function together with a worst-case optimal step size and
contraction factor.
In contrast, a simple counterexample demonstrates that classic \(\EF\) fails without additional
assumptions when local minimizers differ.
Under statistical homogeneity (shared minimizer), we provide a tight analysis for \(\EF\) and
recover the same optimal rate as \(\EFtw\); for deterministic linear compressors we further
obtain sharp guarantees with heterogeneous regularity parameters via an averaged-parameter reduction.

The Lyapunov functions proposed in this analysis could provide insights into the behavior of
these methods under different sets of assumptions and, given the close relationship between
\(\EF\) and \(\EFtw\), could inspire the construction of Lyapunov functions for other
error-compensated algorithms.
A natural next step is to complement these method-specific tight guarantees with
method-agnostic lower bounds, for example lower bounds that depend explicitly on the
communication budget or compression level and apply across broader classes of
compressed distributed algorithms.
These results demonstrate that, when agents share the same smoothness and strong-convexity parameters,
the single-agent setting captures the core worst-case behavior and optimal parameter settings observed in larger systems,
thereby unifying existing theory.

Finally, a number of empirical laws have been formulated about the behavior of \(\EFtw\) with
general heterogeneous regularity parameters, with strong numerical evidence supporting their validity.
Establishing these laws rigorously is left for future work; caution is warranted, as the
dependence of the optimal rate on heterogeneous parameters already appears technically intricate,
even for \(n=2\).

\section*{Acknowledgments}
D.~Berg Thomsen and  A.~Taylor are supported by the European Union (ERC grant CASPER 101162889).
The work of A. Dieuleveut is partly supported by ANR-19-CHIA-0002-01/chaire SCAI, and
Hi!Paris FLAG project, PEPR Redeem. The French government also partly funded this work under
the management of Agence Nationale de la Recherche as part of the ``France 2030'' program,
references ANR-23-IACL-0008 "PR[AI]RIE-PSAI", ANR-23-PEIA-005 (REDEEM project) and ANR-23-IACL-0005.

\section*{Impact Statement}
This paper presents work whose goal is to advance the field of Machine
Learning. There are many potential societal consequences of this work, none
of which the authors feel must be specifically highlighted here.

\bibliography{bib}
\bibliographystyle{icml2026}

\clearpage
\appendix
\onecolumn
\section*{Organization of the Appendix}
\newcommand{\customtoc}[2]{
\item[\ref{#1}] #2 \dotfill \pageref{#1}
}
\begin{itemize}[leftmargin=2em]
    \customtoc{sec:proofs}{Proofs}
    \customtoc{sec:additional_experiments}{Additional Details for the Empirical Laws}
    \customtoc{sec:verification_details}{Verification Details}
\end{itemize}

This appendix provides additional content and details complementing the paper. In particular, \Cref{sec:proofs} provides the complete missing proofs for the main results of the paper. \Cref{sec:additional_experiments} collects additional empirical-law details, including the Lyapunov discovery pipeline, the derivation and root behavior of \Cref{law:rate_n2}, and a comparison with an existing distributed \(\EFtw\) analysis. Finally, \Cref{sec:verification_details} details the verification methodology used to validate the empirical laws.

\newpage
\section{Proofs}
\label{sec:proofs}

\subsection{Proof of \Cref{prop:counterexample_ef}}
\propositionCounterexampleEF*
\label[appendix]{sec:proof_counterexample_ef}
\begin{proof}
  Begin by dealing with the first case, where $\eta < \frac{2}{\mu}$. Use the same
  initialization as in \eqref{eq:x0_cgd}. Set the compression oracle responses on the \(\EF\) inputs
  \(e_0^{(i)}+\eta(f^{(i)})'(x_0)\) to be the scaled analogues of the responses in \eqref{eq:x1_cgd}. Since the initial
  error terms are all zero, the first step will be the same as in \eqref{eq:x1_cgd}. After the
  first round of communication, the error terms are given by
  \begin{equation}
    \begin{aligned}
      e^{(1)}_1 &= e^{(1)}_0 + \eta (f^{(1)})'(x_0) - \C(e^{(1)}_0 + \eta (f^{(1)})'(x_0); \omega_0^{(1)}) = \sqrt{\epsilon} \eta (f^{(1)})'(x_0), \\
      e^{(2)}_1 &= e^{(2)}_0 + \eta (f^{(2)})'(x_0) - \C(e^{(2)}_0 + \eta (f^{(2)})'(x_0); \omega_0^{(2)}) = -\sqrt{\epsilon} \eta (f^{(2)})'(x_0).
    \end{aligned}
  \end{equation}
  Now, set the compression oracles to return the full \(\EF\) inputs at $x_1$ without compression,
  i.e., $m^{(i)}_1 = e_1^{(i)} + \eta(f^{(i)})'(x_1)$ for \(i\in\{1,2\}\). This will result in the following updates
  to the error terms:
  \begin{align*}
    e^{(1)}_2 &= e^{(1)}_1 + \eta (f^{(1)})'(x_1) - \C \left(e^{(1)}_1 + \eta (f^{(1)})'(x_1); \omega_1^{(1)} \right) = 0, \\
    e^{(2)}_2 &= e^{(2)}_1 + \eta (f^{(2)})'(x_1) - \C \left(e^{(2)}_1 + \eta (f^{(2)})'(x_1); \omega_1^{(2)} \right) = 0,
  \end{align*}
  meaning that the error terms are zero after the second step. This also results in the update
  \begin{align*}
    x_2 &= x_1 - \frac{1}{2} \left[e^{(1)}_1 + \eta (f^{(1)})'(x_1) + e^{(2)}_1 + \eta (f^{(2)})'(x_1) \right] \\
    &= x_1 - \frac{\eta}{2} \left[ \sqrt{\epsilon}(\mu x_0 - 1) + (\mu x_1 - 1) - \sqrt{\epsilon} (\mu x_0 + 1) + (\mu x_1 + 1) \right] \\
    &= x_1 - \eta \left(\mu x_1 - \sqrt{\epsilon} \right) = x_0, \\
  \end{align*}
  where the last step follows from \eqref{eq:x0_cgd} and some basic algebra. Since the iterate is back
  at the starting point, with the error terms being zero, the cycle is complete.

  The second case where $\eta > \frac{2}{\mu}$ is given by the converse argument, with the
  starting point $x_0 = -\frac{\eta \sqrt{\epsilon}}{2 - \eta \mu}$ and compression oracles
  that on the first step give the following responses on the \(\EF\) inputs (recall \(e_0^{(i)}=0\)):
  \begin{equation}
    \C(\eta(f^{(1)})'(x_0); \omega_0^{(1)}) = \eta(1 + \sqrt{\epsilon}) (f^{(1)})'(x_0), \qquad
    \C(\eta(f^{(2)})'(x_0); \omega_0^{(2)}) = \eta(1 - \sqrt{\epsilon}) (f^{(2)})'(x_0).
  \end{equation}

  The third case is given by simply considering what would happen if all the compression
  oracles just gave the true gradient updates during all communication rounds. Then there
  would be no compression nor error feedback, making the iteration equivalent to
  distributed gradient descent. The cycle follows from the following computations:
  \begin{align*}
    x_1 &= x_0 - \frac{\eta}{2} \left[(f^{(1)})'(x_0) + (f^{(2)})'(x_0) \right] \\
    &= x_0 - \eta \mu x_0 \\
    &= -x_0,
  \end{align*}
  and,
  \begin{align*}
    x_2 &= x_1 - \frac{\eta}{2} \left[(f^{(1)})'(x_1) + (f^{(2)})'(x_1) \right] \\
    &= x_1 - \eta \mu x_1 \\
    &= -x_1.
  \end{align*}
\end{proof}

\subsection{Proof of \Cref{thm:ef21}}
\theoremEFtw*
\label[appendix]{sec:proof_ef21}
\begin{proof}
  The proof of the announced convergence rate follows. Denote \(g_k^{(i)} \coloneqq \nabla f^{(i)}(x_k)\),
  so that Algorithm~\ref{alg:ef21} writes as
  \[
    x_{k+1} = x_k - \eta \cdot \frac{1}{n} \sum_{i=1}^n d_k^{(i)},
    \qquad
    d_{k+1}^{(i)} = d_k^{(i)} + \C(g_{k+1}^{(i)} - d_k^{(i)}; \omega_k^{(i)}).
  \]
  Consider the following inequalities, and associate with each of them the assigned multiplier:
  \begin{equation*}
    \begin{aligned}
      \begin{aligned}
        I_{\F_{\mu, L}}^{(i,1)} \coloneqq\ &
        f^{(i)}(x_k) - f^{(i)}(x_{k+1})
        + \frac{\|g_{k+1}^{(i)} - g_k^{(i)}\|^2}{2L}
        + \langle g_k^{(i)}, x_{k+1} - x_k\rangle \\
        &\ + \frac{\mu}{2(1 - \mu / L)}\left\|x_k - x_{k+1} - \frac{1}{L}\left(g_k^{(i)} - g_{k+1}^{(i)}\right)\right\|^2 \le 0,
      \end{aligned}
      & \quad & : \lambda_{\EFtw} \\
      \begin{aligned}
        I_{\F_{\mu, L}}^{(i,2)} \coloneqq\ &
        f^{(i)}(x_{k+1}) - f^{(i)}(x_k)
        + \frac{\|g_k^{(i)} - g_{k+1}^{(i)}\|^2}{2L}
        + \langle g_{k+1}^{(i)}, x_k - x_{k+1}\rangle \\
        &\ + \frac{\mu}{2(1 - \mu / L)}\left\|x_{k+1} - x_k - \frac{1}{L}\left(g_{k+1}^{(i)} - g_k^{(i)}\right)\right\|^2 \le 0,
      \end{aligned}
      & \quad & : \lambda_{\EFtw} \\
      I_{\C}^{(i)} \coloneqq \E_{\omega_k^{(i)}}\!\left[\|g_{k+1}^{(i)} - d_k^{(i)} - \C(g_{k+1}^{(i)} - d_k^{(i)}; \omega_k^{(i)})\|^2\right] - \epsilon\,\E_{\omega_k^{(i)}}\!\left[\|g_{k+1}^{(i)} - d_k^{(i)}\|^2\right] \leq 0,
      & \quad & : \nu_{\EFtw},
    \end{aligned}
  \end{equation*}
  where \(i \in [n]\), \(\nu_{\EFtw} \coloneqq 1\), and where \(\lambda_{\EFtw}\) is defined as
  \begin{equation}\label{eq:lambda_ef21_multiagent}
    \lambda_{\EFtw} \coloneqq \frac{\sqrt{\epsilon}}{\eta^\star (L + \mu)} \left[
      (1 - \sqrt{\epsilon})(L - \mu) + (1 + \sqrt{\epsilon}) \sqrt{(L - \mu)^2 + \frac{16 L \mu \sqrt{\epsilon}}{(1 + \sqrt{\epsilon})^2}}
    \right].
  \end{equation}

  Summing the inequalities with their multipliers, the following algebraic identity holds:
  \begin{equation}\label{eq:ef21_multiagent_sos}
    \lambda_{\EFtw} \sum_{i=1}^n (I_{\F_{\mu, L}}^{(i,1)} + I_{\F_{\mu, L}}^{(i,2)}) + \nu_{\EFtw} \sum_{i=1}^n I_{\C}^{(i)} = \E_{\omega_k}\left[\mathcal{V}(\xi_{k+1})\right] - \rho \mathcal{V}(\xi_k) + S,
  \end{equation}
  where the residual \(S\) is a sum of squares defined as
  \[
    S \coloneqq n a S_{\text{mean}} + c S_{\text{var}} + (\rho - \epsilon) S_{\text{mix}},
  \]
  with the components
  \begin{align*}
    S_{\text{mean}} &\coloneqq \left\| \bar d_k + \frac{1}{a} \left[ (\epsilon + b) \bar g_{k+1} - (\rho + b) \bar g_k \right] \right\|^2, \\
    S_{\text{var}} &\coloneqq \sum_{i=1}^n \left\| (g_{k+1}^{(i)} - \bar g_{k+1}) - (g_k^{(i)} - \bar g_k) \right\|^2, \\
    S_{\text{mix}} &\coloneqq \sum_{i=1}^n \left\| (d_k^{(i)} - \bar d_k) + \frac{\epsilon}{\rho - \epsilon} (g_{k+1}^{(i)} - \bar g_{k+1}) - \frac{\rho}{\rho - \epsilon} (g_k^{(i)} - \bar g_k) \right\|^2,
  \end{align*}
  where \(\bar g \coloneqq \frac{1}{n} \sum_{i=1}^n g^{(i)}\) and \(\bar d \coloneqq \frac{1}{n} \sum_{i=1}^n d^{(i)}\) denote the averages, and the coefficients are given by
  \[
    a \coloneqq \rho - \epsilon + \lambda_{\EFtw} \eta^2 \frac{L \mu}{L - \mu}, \quad
    b \coloneqq \frac{\eta \lambda_{\EFtw}}{2} \cdot \frac{L + \mu}{L - \mu}, \quad
    c \coloneqq \frac{\lambda_{\EFtw}}{L - \mu} - \frac{\epsilon \rho}{\rho - \epsilon}.
  \]
  Note that the positivity of \(a\) is guaranteed for \(\eta^\star\), as noted in the
  single-agent proof in~\cite{bergthomsen2025errorfeedback}.
  To see that \(c>0\), set
  \(t \coloneqq \sqrt{(1+\sqrt{\epsilon})^2 + \tfrac{16 L \mu \sqrt{\epsilon}}{(L-\mu)^2}}\) and
  \(u \coloneqq 1-\sqrt{\epsilon}+t\) (so \(u\ge2\)). Then
  \[
    (L-\mu)(\rho-\epsilon)c
    = a_2 (u-2)^2 + a_1 (u-2) + \frac{\sqrt{\epsilon}}{(L+\mu)^2}\!\left[(1+\sqrt{\epsilon}+\epsilon)(L-\mu)^2 + 4\sqrt{\epsilon}\,L\mu\right],
  \]
  where
  \[
    \begin{aligned}
      a_2 &\coloneqq \frac{\sqrt{\epsilon}(1+\sqrt{\epsilon})}{4}
        \left(\frac{L-\mu}{L+\mu}\right)^2, \\
      a_1 &\coloneqq \sqrt{\epsilon}(1+\sqrt{\epsilon})
        \left(\frac{L-\mu}{L+\mu}\right)^2
        + \frac{\epsilon}{2}\left(1+\sqrt{\epsilon}
        -(1-\sqrt{\epsilon})\left(\frac{L-\mu}{L+\mu}\right)^2\right).
    \end{aligned}
  \]
  Both coefficients are positive for \(0<\epsilon<1\) and \(0<\mu<L\), hence \(c>0\).

  Since \(\lambda_{\EFtw} \ge 0\), the weighted sum of inequalities (LHS of~\eqref{eq:ef21_multiagent_sos}) is
  nonpositive. The statement now follows by plugging in \(\eta=\eta^\star\) and \(\rho=\rho_\star\) and
  checking that all coefficients in~\eqref{eq:ef21_multiagent_sos} are nonnegative.

  To show tightness of the bound, consider the case where all agents have the
  same objective function \(f^{(i)} = f\). Initialize
  \(d_0^{(i)} = d_0 = \C(\nabla f(x_0); \omega_0)\) and assume identical compression
  realizations. Then \(d_k^{(i)} = d_k\) for all \(k\), and the algorithm
  updates match the single-agent \(\EFtw\). The single-agent tight lower
  bound thus limits the multi-agent performance. Furthermore, since this
  scenario is a worst-case instance, the optimal step size for the single-agent
  setting is also worst-case optimal for the multi-agent setting. In
  particular, since the dynamics of the identical-functions case are invariant
  to the number of agents \(n\), the worst-case optimal step size must also be
  independent of \(n\).

  The optimality and multi-step tightness of the Lyapunov function follow from the
  theory of linear dynamical systems. For the methods and function classes
  considered, the worst-case contraction is governed by the spectral radius of the
  iteration matrix on quadratic instances. The class of candidate Lyapunov functions matches this
  spectral behavior, ensuring that the single-step analysis is optimal and
  remains tight over multiple iterations (i.e., \(\V_k \leq \rho^k \V_0\) is achievable).
\end{proof}

\subsection{Proof of \Cref{cor:ef21_linear}}
\corollaryEFtwLinear*
\label[appendix]{sec:proof_ef21_linear}
\begin{proof}
  Recall \(f(x) \coloneq \frac{1}{n}\sum_{i=1}^n f^{(i)}(x)\), and denote
  \(\bar g_k \coloneq \nabla f(x_k) = \frac{1}{n}\sum_{i=1}^n \nabla f^{(i)}(x_k)\).
  Since each \(f^{(i)}\) is \(L^{(i)}\)-smooth and \(\mu^{(i)}\)-strongly convex,
  \(f\) is \(\bar L\)-smooth and \(\bar \mu\)-strongly convex with
  \(\bar L = \frac{1}{n}\sum_i L^{(i)}\) and \(\bar \mu = \frac{1}{n}\sum_i \mu^{(i)}\).
  Consequently, the standard interpolation (cocoercivity) inequalities for
  \(\mathcal{F}_{\bar\mu,\bar L}\) apply directly to \(\bar g_k\) and \(\bar g_{k+1}\).

  The averaged dynamics satisfy
  \[
    x_{k+1} = x_k - \eta\, \bar d_k,
  \]
  and, due to additivity and positive homogeneity of \(\C\),
  \begin{equation}
    \begin{aligned}
      \bar d_{k+1}
      &= \frac{1}{n}\sum_{i=1}^n d_{k+1}^{(i)}
      = \frac{1}{n}\sum_{i=1}^n \left(d_k^{(i)} + \C\!\left(g_{k+1}^{(i)} - d_k^{(i)}\right)\right) \\
      &= \bar d_k + \C\!\left(\frac{1}{n}\sum_{i=1}^n \left(g_{k+1}^{(i)} - d_k^{(i)}\right)\right)
      = \bar d_k + \C(\bar g_{k+1} - \bar d_k).
    \end{aligned}
  \end{equation}
  Therefore, the barred variables evolve as a single-agent \(\EFtw\) instance
  on \(f\) with parameters \((\bar \mu,\bar L)\).
  The conclusion follows by applying Theorems~1--2 in \citet{bergthomsen2025errorfeedback} to this
  averaged single-agent instance. The optimal step size and rate are obtained by substituting
  \(\kappa_\Sigma = \bar L / \bar \mu = \sum_i L^{(i)} / \sum_i \mu^{(i)}\) into \eqref{eq:opt_rate}.
  Theorem~3 in \citet{richtarik_ef21_2021} establishes that $\EF$ and $\EFtw$ are equivalent under
  deterministic, additive, positively homogeneous compressors, so the same rate applies to $\EF$.
\end{proof}

\subsection{Proof of \Cref{thm:ef}}
\theoremEF*
\label[appendix]{sec:proof_ef}
\begin{proof}
  The proof of the announced convergence rate follows. Algorithm~\ref{alg:ef} can be rewritten as
  \[
    x_{k+1} = x_k - \frac{1}{n} \sum_{i=1}^n m_k^{(i)},
    \qquad
    e_{k+1}^{(i)} = e_k^{(i)} + \eta \nabla f^{(i)}(x_k) - m_k^{(i)},
  \]
  where \(g_k^{(i)} \coloneqq \nabla f^{(i)}(x_k)\),
  \(u_k^{(i)} \coloneqq e_k^{(i)} + \eta g_k^{(i)}\), and
  \(m_k^{(i)} = \C(u_k^{(i)}; \omega_k^{(i)})\).
  For \(z\in\{e,g,u,m\}\), write \(\bar z_k \coloneqq \frac1n\sum_{i=1}^n z_k^{(i)}\).
  Consider the following inequalities:
  \begin{equation*}
    \begin{aligned}
      \begin{aligned}
        I_{\F_{\mu, L}}^{(i,1)} \coloneqq\ &
        f^{(i)}(x_k) - f^{(i)}(x^\star)
        - \langle \nabla f^{(i)}(x_k), x_k - x^\star \rangle
        + \frac{1}{2L} \|\nabla f^{(i)}(x_k)\|^2 \\
        &\ + \frac{\mu}{2(1 - \mu / L)} \left\|x_k - x^\star - \frac{1}{L} \nabla f^{(i)}(x_k)\right\|^2 \leq 0,
      \end{aligned}
      & \quad & : \lambda_{\EF} \\
      \begin{aligned}
        I_{\F_{\mu, L}}^{(i,2)} \coloneqq\ &
        f^{(i)}(x^\star) - f^{(i)}(x_k)
        + \frac{1}{2L} \|\nabla f^{(i)}(x_k)\|^2 \\
        &\ + \frac{\mu}{2(1 - \mu / L)} \left\|x_k - x^\star - \frac{1}{L} \nabla f^{(i)}(x_k)\right\|^2 \leq 0,
      \end{aligned}
      & \quad & : \lambda_{\EF} \\
      \begin{aligned}
        I_{\C}^{(i)} \coloneqq\ &
        (1-\epsilon)\|u_k^{(i)}\|^2
        - 2\left\langle u_k^{(i)}, \E_{\omega_k^{(i)}}\!\left[\C(u_k^{(i)}; \omega_k^{(i)})\right]\right\rangle \\
        &\ + \E_{\omega_k^{(i)}}\!\left[\|\C(u_k^{(i)}; \omega_k^{(i)})\|^2\right] \le 0,
      \end{aligned}
      & \quad & : \nu_{\EF},
    \end{aligned}
  \end{equation*}
  where \(i \in [n]\), \(\nu_{\EF} \coloneqq 1/\sqrt{\epsilon}\), and \(\lambda_{\EF}\) is defined as
  \begin{equation}\label{eq:lambda_ef_multiagent}
    \lambda_{\EF} \coloneqq \frac{\eta^\star}{L + \mu} \left[
      (1 - \sqrt{\epsilon})(L - \mu) + (1 + \sqrt{\epsilon}) \sqrt{(L - \mu)^2 + \frac{16 L \mu \sqrt{\epsilon}}{(1 + \sqrt{\epsilon})^2}}
    \right].
  \end{equation}
  Set \(\rho \coloneqq \rho_\star\) and
  \(a \coloneqq (\rho-\sqrt{\epsilon})(1+\sqrt{\epsilon})/\sqrt{\epsilon}\).
  Summing these inequalities with their multipliers yields the algebraic identity:
  \begin{equation}\label{eq:ef_multiagent_sos}
    \lambda_{\EF} \sum_{i=1}^n (I_{\F_{\mu, L}}^{(i,1)} + I_{\F_{\mu, L}}^{(i,2)}) + \nu_{\EF} \sum_{i=1}^n I_{\C}^{(i)} = \E_{\omega_k}\left[\mathcal{V}(\xi_{k+1})\right] - \rho \mathcal{V}(\xi_k) + S,
  \end{equation}
  where the residual \(S\) is given by
  \[
    S \coloneqq n a S_{\text{mean}} + c S_{\text{var}} + (\rho - \sqrt{\epsilon}) S_{\text{mix}} + (\nu_{\EF} - 1) S_{\text{mix-comp}},
  \]
  with the components
  \begin{align*}
    S_{\text{mean}} &\coloneqq \left\| \bar e_k - \frac{\rho-1}{a} (x_k - x^\star) + \frac{2(\sqrt{\epsilon}-1)}{a(L+\mu)} \bar g_k \right\|^2, \\
    S_{\text{var}} &\coloneqq \sum_{i=1}^n \left\| g_k^{(i)} - \bar g_k \right\|^2, \\
    S_{\text{mix}} &\coloneqq \sum_{i=1}^n \left\| (e_k^{(i)} - \bar e_k) - \frac{\sqrt{\epsilon}\, \eta}{\rho - \sqrt{\epsilon}} (g_k^{(i)} - \bar g_k) \right\|^2, \\
    S_{\text{mix-comp}} &\coloneqq \sum_{i=1}^n \E_{\omega_k}\!\left[
      \left\| (u_k^{(i)}-\bar u_k) - (m_k^{(i)}-\bar m_k) \right\|^2
    \right].
  \end{align*}
  As in the single-agent proof, the positivity of
  \(\rho - \sqrt{\epsilon}\) (hence \(a\)) for \(\eta^\star\) follows directly from the single-agent analysis.
  Let \(s \coloneqq \sqrt{\epsilon}\) and
  \(t \coloneqq \sqrt{(L-\mu)^2 + \tfrac{16 L \mu s}{(1+s)^2}}\). This also gives
  \(\nu_{\EF}-1 = \tfrac{1-s}{s} > 0\), and
  \[
    c = \frac{2(1-s)\big((1-s)(L-\mu)+(1+s)t\big)}{(1+s)^2 (L-\mu) (L+\mu)^2} > 0.
  \]

  To show tightness of the bound, consider the case where all agents have the
  same objective function \(f^{(i)} = f\). With initialization
  \(e_0^{(i)} = e_0 = 0\) and identical compression realizations,
  \(e_k^{(i)} = e_k\) for all \(k\), and the algorithm updates match the
  single-agent \(\EF\). The single-agent tight lower bound thus limits the
  multi-agent performance. Furthermore, since this scenario is a worst-case
  instance, the optimal step size for the single-agent setting is also
  worst-case optimal for the multi-agent setting. In particular, since the
  dynamics of the identical-functions case are invariant to the number of agents
  \(n\), the worst-case optimal step size must also be independent of \(n\).

  The optimality of the Lyapunov function and tightness over multiple iterations follow from the same argument as in the proof of \Cref{thm:ef21}.
\end{proof}

\newpage
\section{Additional Details for the Empirical Laws}
\label{sec:additional_experiments}

\subsection{Derivation of the \texorpdfstring{\(n=2\)}{n=2} rate law}
The cubic polynomial in \Cref{law:rate_n2} was identified through an iterative workflow
combining numerical PEP experiments with symbolic algebra (using Mathematica). We started from 
the empirical optimal step size and Lyapunov function in Empirical Laws~\ref{law:opt_step}
and~\ref{law:lyap_ef21}, and then
formed the dual problem of the corresponding performance estimation problem using that
Lyapunov function as the error metric.

Large parameter sweeps showed that the optimal dual variables (Lagrange multipliers) differ
substantially from the single-agent case. The associated rate is also slightly different.
This ruled out a direct transfer of the
single-agent proof structure and required searching for new closed-form relations. We tried
several structural constraints, reparameterizations, and symbolic-regression heuristics to
infer explicit formulas for the multipliers, without success.

Subsequent experiments indicated that the dual linear matrix inequality (LMI) is
rank-deficient in the relevant regimes. Exploiting this structure revealed relations between
the multipliers and the convergence rate, which eventually led to the cubic characterization
in \Cref{law:rate_n2} for the \(n=2\) case.

Because the LMI is highly intricate, this discovery required multiple rounds of
reparameterization and simplification rather than a single linear derivation.

\subsection{Lyapunov-function discovery pipeline}
The Lyapunov structures used in these results were found through a numerical
certificate-search workflow. Starting with the two-worker case, we solved the
SDP-based Lyapunov search using the log-det heuristic~\citep{fazel2003log} and manually introduced
sparsification constraints to identify a minimal structure. Since the remaining
coefficients were not unique, we fixed a normalization and probed each coefficient's
feasible range by alternately maximizing and minimizing it while leaving the other
coefficients free. These ranges revealed a stable limiting pattern as
\(\epsilon \to 1\), matching the simple coefficients used in the final Lyapunov
functions. The resulting candidates were then checked numerically across additional
parameter sweeps and larger worker counts.

\subsection{Roots of the \texorpdfstring{\(n=2\)}{n=2} rate polynomial}
Closed-form expressions for the roots exist in principle, but they are too cumbersome to be
informative in the present context and would significantly increase the length of the paper.
\Cref{fig:poly_bifurcation} contains plots of the roots of the cubic polynomial defined 
in \Cref{law:rate_n2}. The bifurcation plots highlight that the roots vary nontrivially 
with \(\epsilon\) and heterogeneity, so there is no single root expression that can be 
uniformly exploited to simplify the cubic (e.g., by direct factorization).
In particular, no branch yields a compact expression that remains interpretable across
all heterogeneity settings.
\begin{figure*}[t]
  \centering
  \includegraphics[width=\textwidth]{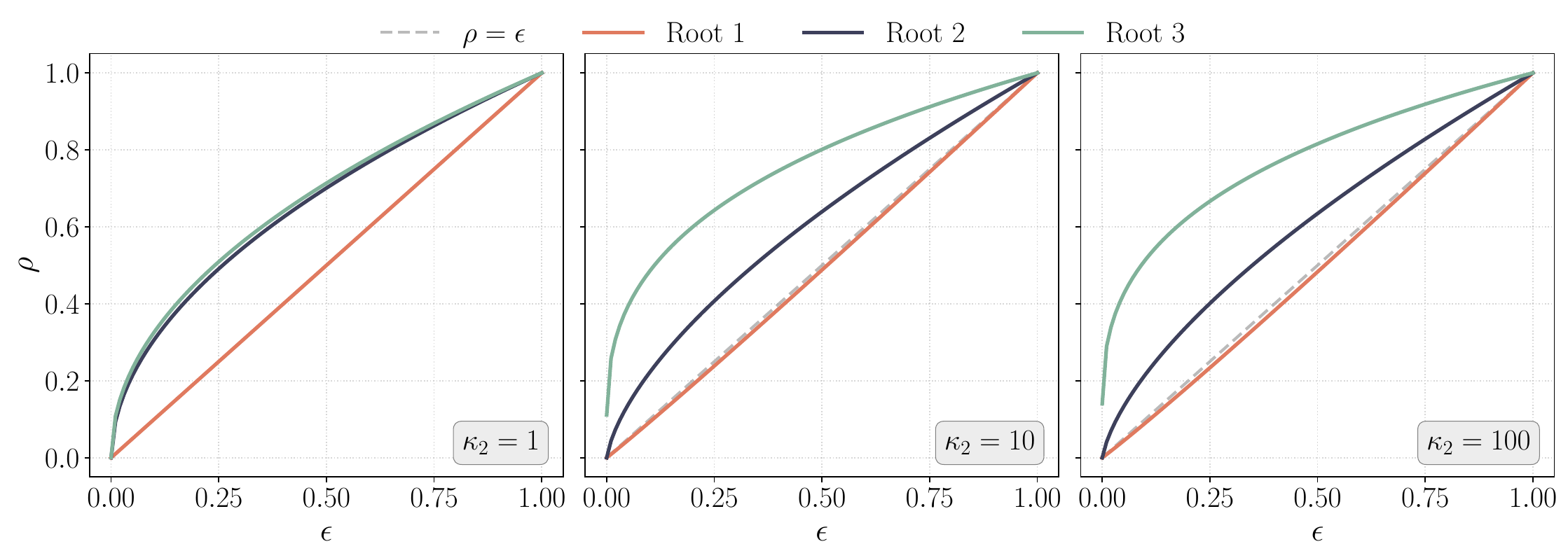}
  \caption{Bifurcation plot of the three roots of the cubic polynomial in
    \Cref{law:rate_n2} as \(\epsilon\) varies. Parameters are
    \(L^{(1)} = L^{(2)} = 1\), \(\mu^{(1)} = 0.9\), and
  \(\mu^{(2)} \in \{1, 0.1, 0.01\}\) from left to right. The dashed curve is
  \(\rho=\epsilon\), shown as a reference to indicate that no root branch collapses
  to this expression uniformly across heterogeneity settings.}
  \label{fig:poly_bifurcation}
\end{figure*}

\subsection{Comparison with an existing \texorpdfstring{\(\EFtw\)}{EF21} analysis}
\Cref{fig:richtarik_comparison} compares the relative iteration count predicted
by \Cref{law:rate_n2} against the distributed \(\EFtw\) rate from
\citet{richtarik_ef21_2021}. The \(\EFtw\) baseline uses the largest step size
allowed by that analysis. The vertical axis is
\(\log(\rho_{\EFtw})/\log(\rho_\star)\), where \(\rho_\star\) is the
rate in \Cref{law:rate_n2}; a value of \(1/2\), for example, means that
\Cref{law:rate_n2} reaches a fixed target accuracy in half as many iterations as
the corresponding \(\EFtw\) bound. The corresponding script in the
\href[pdfnewwindow=true]{https://github.com/DanielBergThomsen/distributed-error-feedback/blob/main/experiments.py}{public repository}
can be reused with other published rates or standard tuning rules.
\begin{figure*}[t]
  \centering
  \includegraphics[width=.9\textwidth]{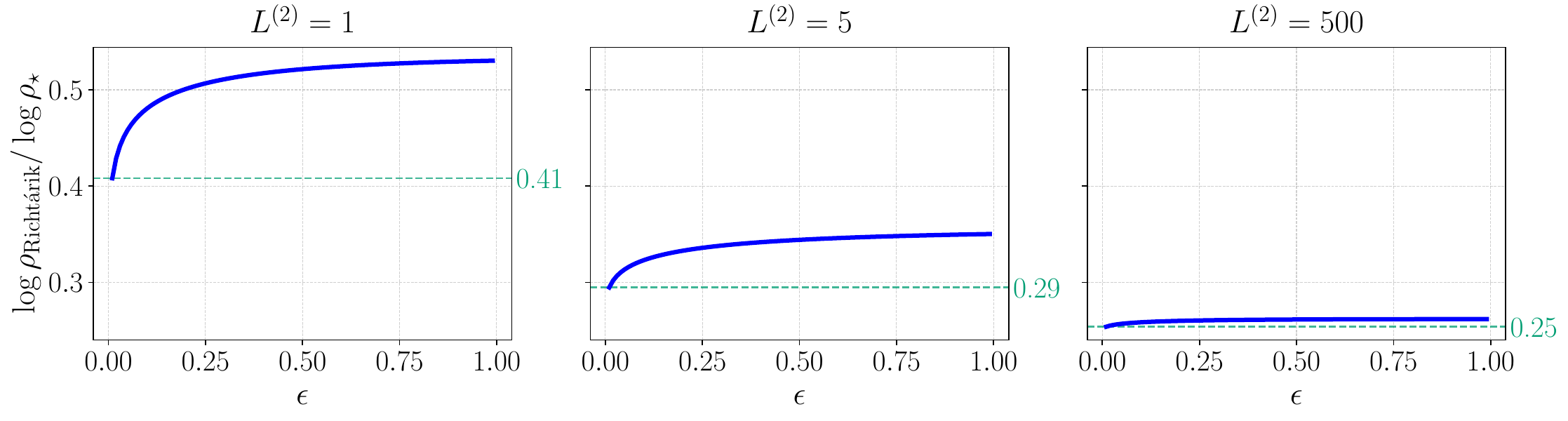}
  \caption{Iteration-complexity comparison between the rate \(\rho_\star\)
    from Empirical Law~\ref{law:rate_n2} and the distributed \(\EFtw\) rate from
    \citet{richtarik_ef21_2021}. The vertical axis is
    \(\log(\rho_{\EFtw})/\log(\rho_\star)\), so smaller values indicate
    fewer iterations for the empirical-law rate to reach a fixed target accuracy.
    Here \(n=2\), \(\mu^{(1)}=\mu^{(2)}=0.5\), \(L^{(1)}=1\), and
    \(L^{(2)}\in\{1,5,500\}\) from left to right.}
  \label{fig:richtarik_comparison}
\end{figure*}

\newpage
\section{Verification Details}
\label{sec:verification_details}
Validation of the empirical laws was performed using the Performance Estimation Problem (PEP)
framework~\cite{drori2014contributions,taylor2017smooth}. The verifications follow the SDP-based Lyapunov identification
techniques of \citet{taylor2018lyapunov,upadhyaya2023automated}, with the empirical-law certificates implemented directly as
CVXPY Lyapunov LMIs in the public notebooks. Some feasibility
checks were numerically delicate on highly heterogeneous instances. We therefore accepted only clean solver outputs as
certificates: outputs with known numerical warning patterns were treated as uncertified in the corresponding feasibility check.
MOSEK~\citep{mosek} was used for bisection and screening, while SDPA~\citep{yamashita2010sdpa7} was used as an
independent high-precision confirmation step for apparent violations of the empirical laws.
Unless stated otherwise, strict MOSEK runs used one thread and interior-point feasibility, gap, and barrier-reduction tolerances \(10^{-10}\);
SDPA confirmation runs used \texttt{epsilonStar}$=\texttt{epsilonDash}=10^{-11}$, \texttt{mpfPrecision}$=4096$, and
\texttt{maxIteration}$=2000$, together with the fixed SDPA tuning parameters reported in the public notebooks.

Across all sweeps, worker permutations were deduplicated because the LMIs are invariant under agent relabeling. For each raw
regularity configuration, the smoothness constants were normalized by \(L_{\max}=\max_i L^{(i)}\), and the corresponding
\(\mu^{(i)}\) values were set through \(\mu^{(i)}=L^{(i)}/\kappa^{(i)}\) after this normalization. This common rescaling only
changes the step-size units and leaves the contraction-factor feasibility unchanged. To test tightness or optimality,
we attempted a \(1\%\) improvement target, denoted \(\rho_{\mathrm{imp}}=0.99\,\rho\). An apparent improvement was treated as a counterexample only
when this improved target was confirmed by the prescribed solver sequence without numerical warnings. No such cleanly confirmed counterexamples were found in any of the
experiments below.

\subsection{Verification of Optimal Step Size (Empirical Law \ref{law:opt_step})}
The optimal-step experiment was run separately for \(\EF\) and \(\EFtw\). The \(\EF\) checks used the homogeneous-worker
normalization corresponding to the shared-minimizer setting of \Cref{thm:ef}. Each sweep used \(10\) linearly spaced values of
\(\epsilon \in [0.05,0.95]\). For \(n=2\), we used \(L^{(i)} \in \{1,100,1000\}\) and
\(\kappa^{(i)} \in \{2,100,1000\}\), giving \(450\) configurations per method. For \(n=3\), we used
\(L^{(i)} \in \{1,1000\}\) and \(\kappa^{(i)} \in \{2,100,1000\}\), giving \(560\) configurations per method.
For \(n=4\), we used \(L^{(i)} \in \{1,1000\}\) and \(\kappa^{(i)} \in \{2,1000\}\), giving \(350\)
configurations per method.

For each instance, the theoretical step size \(\eta^\star\) from \Cref{law:opt_step} was evaluated in the full Lyapunov LMI.
For \(n=2\), the reference rate \(\rho_\star\) was taken from the cubic formula in \Cref{law:rate_n2} and then certified at
\(\eta^\star\). For \(n=3,4\), \(\rho_\star\) was computed by bisection on \([0,1]\) with tolerance \(10^{-7}\).

The candidate step sizes are not drawn from the whole interval \([0, 2n/\sum_i(L^{(i)}+\mu^{(i)})]\). Instead, the search is 
restricted to the band forced by the exact-compression mode. Indeed, if the compressor returns the uncompressed gradient and the 
error feedback state is zero, both \(\EF\) and \(\EFtw\) reduce to gradient descent on the averaged objective 
\(\bar f \coloneqq \tfrac1n \sum_i f^{(i)}\), whose smoothness and strong-convexity parameters are 
\(\bar L \coloneqq \tfrac1n \sum_i L^{(i)}\) and \(\bar \mu \coloneqq \tfrac1n \sum_i \mu^{(i)}\). The known rate of gradient 
descent on \(\mathcal F_{\bar \mu,\bar L}\) is \(\max_{\lambda \in [\bar \mu,\bar L]} (1-\eta \lambda)^2\), so any candidate 
achieving rate \(\rho_\star\) on the full problem class must satisfy
\[
  \max_{\lambda \in [\bar \mu,\bar L]} (1-\eta \lambda)^2 \le \rho_\star,
  \qquad\text{hence}\qquad
  \eta \in \left[\frac{1-\sqrt{\rho_\star}}{\bar \mu}, \frac{1+\sqrt{\rho_\star}}{\bar L}\right].
\]
A scan of \(40\) equally spaced values was then performed in this interval, clipped to
\([0, 2n/\sum_i(L^{(i)}+\mu^{(i)})]\).

Each candidate step size was first screened with MOSEK at \(\rho_{\mathrm{imp}} = 0.99\,\rho_\star\). Candidates that were not
certified feasible at this target were discarded. For the remaining candidates, the candidate rate was recomputed by bisection;
only candidates still improving on \(\rho_\star\) by at least \(1\%\) were sent to SDPA for high-precision confirmation. This two-stage
procedure avoids false positives from near-boundary numerical noise while keeping the exhaustive search manageable.

\subsection{Verification of Lyapunov Function Structure (Empirical Law \ref{law:lyap_ef21})}
This experiment was run for \(n\in\{2,3,4\}\), with \(10\) linearly spaced values of \(\epsilon \in [0.05,0.95]\).
For each worker, \((L^{(i)},\kappa^{(i)})\) was chosen from
\(\{1,100,1000\}\times\{2,100,1000\}\). This gives \(450\), \(1650\), and \(4950\) configurations for
\(n=2,3,4\), respectively.

At the conjectured step size \(\eta^\star\), bisection on \([0,1]\) with tolerance \(10^{-7}\) was used to compute the rate
\(\rho_{\mathrm{simp}}\) certified by the simplified Lyapunov structure in \Cref{law:lyap_ef21}. The full \(\EFtw\)
Lyapunov class was then tested at \(\rho_{\mathrm{imp}}=0.99\,\rho_{\mathrm{simp}}\). The first full-class check used
MOSEK with strict feasibility and gap tolerances. Whenever this check suggested an improvement, the simplified structure was
directly retested at \(\rho_{\mathrm{imp}}\) and at a further \(10^{-4}\) relative margin; only remaining candidates were
confirmed with a high-precision SDPA solve. No SDPA-confirmed full-class improvement over the simplified structure was found.

\subsection{Verification of $n=2$ Rate (Empirical Law \ref{law:rate_n2})}
For the \(n=2\) rate law, the sweep used \(50\) linearly spaced values of \(\epsilon \in [0.05,0.95]\). For each worker,
the unscaled parameters were drawn from \(L^{(i)} \in \{1,10,100,1000\}\) and
\(\kappa^{(i)} \in \{2,10,100,1000,10000\}\), giving \(210\) unordered regularity configurations and
\(10500\) epsilon-configuration pairs. The candidate rate \(\rho_\star\) was taken as the largest valid real root of the cubic in
\Cref{law:rate_n2}, and the step size was set to \(\eta^\star\).

The empirical law concerns \(\EFtw\); as an auxiliary check, the same candidate was also tested for \(\EF\) using the
homogeneous-worker normalization. For each method, the pair \((\rho_\star,\eta^\star)\) was certified by the corresponding
Lyapunov LMI using MOSEK with strict \(10^{-10}\) feasibility and gap tolerances. As a local sharpness check, the stronger
target \((0.99\,\rho_\star,\eta^\star)\) was also tested and was not certified without numerical warnings on any instance.

\subsection{Verification of \(\EC\) Tuning (Empirical Law \ref{law:econtrol_tuning})}
For \(\EC\), the experiment considered \(n \in \{2,3,4\}\), with \(10\) linearly spaced values of
\(\epsilon \in [0.05,0.95]\), \(L^{(i)} \in \{1,100,1000\}\), and
\(\kappa^{(i)} \in \{2,100,1000\}\). This gives \(450\), \(1650\), and \(4950\) configurations for
\(n=2,3,4\), respectively. For each instance, the conjectured tuning
\((\eta,\gamma)=(0,\gamma^\star_{\EC})\) was evaluated by bisection on the full \(\EC\) Lyapunov LMI with tolerance
\(10^{-7}\), using MOSEK with strict feasibility and gap tolerances.

To search for counterexamples, we tested whether any grid point in the search region could certify a rate at least \(1\%\) smaller than the
reference rate. Writing this target as \(\rho_{\mathrm{imp}}=0.99\,\rho_\star\), the notebook scanned a \(10\times 10\)
grid. The \(\eta\)-grid was drawn in the unnormalized scale over
\[
  \eta \in \left[0,\frac{2n}{\sum_i(L^{(i)}_{\mathrm{raw}}+\mu^{(i)}_{\mathrm{raw}})}\right],
\]
where \(\mu^{(i)}_{\mathrm{raw}}=L^{(i)}_{\mathrm{raw}}/\kappa^{(i)}\). The \(\gamma\)-grid used the exact-compression
band at the improved target,
\[
  \gamma \in
  \left[
    \frac{1-\sqrt{\rho_{\mathrm{imp}}}}{\bar\mu},
    \min\left\{\gamma_{\max},\frac{1+\sqrt{\rho_{\mathrm{imp}}}}{\bar L}\right\}
  \right],
\]
where \(\bar L\) and \(\bar\mu\) are computed after max-\(L\) normalization and
\(\gamma_{\max}=L_{\max}\,2n/\sum_i(L^{(i)}_{\mathrm{raw}}+\mu^{(i)}_{\mathrm{raw}})\). Empty
\(\gamma\)-bands were skipped. Each candidate was first screened on the full \(\EC\) LMI with MOSEK at
\(\rho_{\mathrm{imp}}\); any apparent improvement was then re-evaluated by bisection and confirmed with a
high-precision SDPA solve. No SDPA-confirmed improvement was found in any tested instance.
\end{document}

%% file: preamble.tex
\usepackage{enumitem}
\usepackage{mathtools}
\usepackage{amsmath,amssymb, amsthm}
\usepackage{array}
\usepackage{makecell}
\usepackage{multirow}
\usepackage{tabularray}
\usepackage[export]{adjustbox}
\usepackage{placeins}

\newcommand{\defeq}{\coloneq}
\newcommand{\F}{\mathcal{F}}
\newcommand{\Fml}{\F_{\mu, L}}
\newcommand{\C}{\mathcal{C}}
\newcommand{\V}{\mathcal{V}}
\newcommand{\R}{\mathbb{R}}
\newcommand{\E}{\mathbb{E}}

\newcolumntype{L}[1]{>{\raggedright\arraybackslash}m{#1}}
\newcolumntype{C}[1]{>{\centering\arraybackslash}m{#1}}

\newcommand{\EFtw}{\mathrm{EF^{21}}}
\newcommand{\EF}{\mathrm{EF}}
\newcommand{\CGD}{\mathrm{CGD}}
\newcommand{\EC}{\mathrm{EControl}}
\newcommand{\xs}{x_\star}

\newcommand{\rhostar}[1]{\rho_\star\!\left(#1\right)}
\newcommand{\M}{\mathcal{M}}

\newcommand{\Mfxxi}[3]{\M\!\left(#1, #2; #3\right)}
\newcommand{\SO}{\mathbb{S}}
\DeclareMathOperator{\tr}{Tr}
\newsavebox{\badgebox}
\newsavebox{\badgealignbox}
\DeclareRobustCommand{\alignedbadge}[1]{%
  \sbox{\badgebox}{\includegraphics[height=1.00em]{#1}}%
  \sbox{\badgealignbox}{X}%
  \raisebox{\dimexpr(\ht\badgealignbox-\dp\badgealignbox-\ht\badgebox+\dp\badgebox)/2\relax}{\usebox{\badgebox}}%
}
\newcommand{\PEPbadge}{\alignedbadge{PEP.pdf}}
\newcommand{\CASbadge}{\alignedbadge{CAS.pdf}}

\usepackage{hyperref}
\usepackage{url}
\usepackage{cuted}
\usepackage[capitalize,noabbrev]{cleveref}

\crefname{figure}{figure}{figures}
\Crefname{figure}{Figure}{Figures}
\crefname{table}{table}{tables}
\Crefname{table}{Table}{Tables}

\crefname{section}{section}{sections}
\Crefname{section}{Section}{Sections}
\crefname{subsection}{subsection}{subsections}
\Crefname{subsection}{Subsection}{Subsections}

\crefname{equation}{equation}{equations}
\Crefname{equation}{Eq.}{Eqs.}

\crefname{theorem}{theorem}{theorems}
\Crefname{theorem}{Theorem}{Theorems}
\crefname{lemma}{lemma}{lemmas}
\Crefname{lemma}{Lemma}{Lemmas}
\crefname{proposition}{proposition}{propositions}
\Crefname{proposition}{Proposition}{Propositions}
\crefname{corollary}{corollary}{corollaries}
\Crefname{corollary}{Corollary}{Corollaries}
\crefname{conjecture}{conjecture}{conjectures}
\Crefname{conjecture}{Conjecture}{Conjectures}
\crefname{definition}{definition}{definitions}
\Crefname{definition}{Definition}{Definitions}
\crefname{remark}{remark}{remarks}
\Crefname{remark}{Remark}{Remarks}
\crefname{example}{example}{examples}
\Crefname{example}{Example}{Examples}
\crefname{assumption}{assumption}{assumptions}
\Crefname{assumption}{Assumption}{Assumptions}

\crefname{algorithm}{algorithm}{algorithms}
\Crefname{algorithm}{Algorithm}{Algorithms}

\crefname{appendix}{appendix}{appendices}
\Crefname{appendix}{Appendix}{Appendices}

\usepackage{thmtools}
\usepackage{thm-restate}

%% file: main.bbl
\begin{thebibliography}{50}
\providecommand{\natexlab}[1]{#1}
\providecommand{\url}[1]{\texttt{#1}}
\expandafter\ifx\csname urlstyle\endcsname\relax
  \providecommand{\doi}[1]{doi: #1}\else
  \providecommand{\doi}{doi: \begingroup \urlstyle{rm}\Url}\fi

\bibitem[Alistarh et~al.(2017)Alistarh, Grubic, Li, Tomioka, and
  Vojnovic]{alistarh_qsgd_2017}
Alistarh, D., Grubic, D., Li, J., Tomioka, R., and Vojnovic, M.
\newblock {QSGD}: {Communication}-{Efficient} {SGD} via {Gradient}
  {Quantization} and {Encoding}.
\newblock \emph{Advances in Neural Information Processing Systems (NeurIPS)},
  30, 2017.

\bibitem[Alistarh et~al.(2018)Alistarh, Hoefler, Johansson, Konstantinov,
  Khirirat, and Renggli]{alistarh_convergence_2018}
Alistarh, D., Hoefler, T., Johansson, M., Konstantinov, N., Khirirat, S., and
  Renggli, C.
\newblock The {Convergence} of {Sparsified} {Gradient} {Methods}.
\newblock \emph{Advances in Neural Information Processing Systems (NeurIPS)},
  31, 2018.

\bibitem[Berg~Thomsen et~al.(2025)Berg~Thomsen, Taylor, and
  Dieuleveut]{bergthomsen2025errorfeedback}
Berg~Thomsen, D., Taylor, A., and Dieuleveut, A.
\newblock Tight analyses of first-order methods with error feedback.
\newblock \emph{Advances in Neural Information Processing Systems (NeurIPS)},
  2025.

\bibitem[Beznosikov et~al.(2023)Beznosikov, Horváth, Richtárik, and
  Safaryan]{beznosikov_biased_2020}
Beznosikov, A., Horváth, S., Richtárik, P., and Safaryan, M.
\newblock On {Biased} {Compression} for {Distributed} {Learning}.
\newblock \emph{Journal of Machine Learning Research}, 24\penalty0
  (276):\penalty0 1--50, 2023.

\bibitem[Chilimbi et~al.(2014)Chilimbi, Suzue, Apacible, and
  Kalyanaraman]{chilimbi_project_2014}
Chilimbi, T., Suzue, Y., Apacible, J., and Kalyanaraman, K.
\newblock Project adam: Building an efficient and scalable deep learning
  training system.
\newblock In \emph{{USENIX} {Symposium} on {Operating} {Systems} {Design} and
  {Implementation} ({OSDI} 14)}, 2014.

\bibitem[Condat et~al.(2022)Condat, Yi, and Richt{\'a}rik]{condat2022ef}
Condat, L., Yi, K., and Richt{\'a}rik, P.
\newblock Ef-bv: A unified theory of error feedback and variance reduction
  mechanisms for biased and unbiased compression in distributed optimization.
\newblock \emph{Advances in Neural Information Processing Systems},
  35:\penalty0 17501--17514, 2022.

\bibitem[Cutler(1952)]{cutler1952differential}
Cutler, C.~C.
\newblock Differential quantization of communication signals.
\newblock US Patent 2,605,361, July 1952.
\newblock URL \url{https://patents.google.com/patent/US2605361A/en}.
\newblock Filed June 29, 1950; issued July 29, 1952.

\bibitem[Drori(2014)]{drori2014contributions}
Drori, Y.
\newblock \emph{Contributions to the Complexity Analysis of Optimization
  Algorithms}.
\newblock PhD thesis, Tel-Aviv University, 2014.

\bibitem[Egger et~al.(2025)Egger, Bitar, Wachter-Zeh, Weinberger, and
  Gunduz]{egger2025bicompfl}
Egger, M., Bitar, R., Wachter-Zeh, A., Weinberger, N., and Gunduz, D.
\newblock Bicompfl: Bi-directional compression for stochastic federated
  learning.
\newblock In \emph{ICML 2025 Workshop on Machine Learning for Wireless
  Communication and Networks (ML4Wireless)}, 2025.

\bibitem[Fatkhullin et~al.(2021)Fatkhullin, Sokolov, Gorbunov, Li, and
  Richtárik]{fatkhullin_ef21_2021}
Fatkhullin, I., Sokolov, I., Gorbunov, E., Li, Z., and Richtárik, P.
\newblock {EF21} with {Bells} \& {Whistles}: {Practical} {Algorithmic}
  {Extensions} of {Modern} {Error} {Feedback}, October 2021.
\newblock arXiv:2110.03294 [cs, math].

\bibitem[Fatkhullin et~al.(2023)Fatkhullin, Tyurin, and
  Richt{\'a}rik]{fatkhullin2023momentum}
Fatkhullin, I., Tyurin, A., and Richt{\'a}rik, P.
\newblock Momentum provably improves error feedback!
\newblock \emph{Advances in Neural Information Processing Systems (NeurIPS)},
  2023.

\bibitem[Fatkhullin et~al.(2025)Fatkhullin, Sokolov, Gorbunov, Li, and
  Richt{\'a}rik]{fatkhullin2025ef21}
Fatkhullin, I., Sokolov, I., Gorbunov, E., Li, Z., and Richt{\'a}rik, P.
\newblock Ef21 with bells \& whistles: Six algorithmic extensions of modern
  error feedback.
\newblock \emph{Journal of Machine Learning Research}, 26\penalty0
  (189):\penalty0 1--50, 2025.

\bibitem[Fazel et~al.(2003)Fazel, Hindi, and Boyd]{fazel2003log}
Fazel, M., Hindi, H., and Boyd, S.~P.
\newblock Log-det heuristic for matrix rank minimization with applications to
  {Hankel} and {Euclidean} distance matrices.
\newblock In \emph{American Control Conference (ACC)}, 2003.

\bibitem[Gao et~al.(2023)Gao, Islamov, and Stich]{gao2023econtrol}
Gao, Y., Islamov, R., and Stich, S.
\newblock Econtrol: Fast distributed optimization with compression and error
  control.
\newblock \emph{arXiv preprint arXiv:2311.05645}, 2023.

\bibitem[Gorbunov et~al.(2020)Gorbunov, Kovalev, Makarenko, and
  Richtarik]{gorbunov_linearly_2020}
Gorbunov, E., Kovalev, D., Makarenko, D., and Richtarik, P.
\newblock Linearly {Converging} {Error} {Compensated} {SGD}.
\newblock In \emph{Advances in {Neural} {Information} {Processing} {Systems}
  ({NeurIPS})}, 2020.

\bibitem[Goujaud et~al.(2024)Goujaud, Moucer, Glineur, Hendrickx, Taylor, and
  Dieuleveut]{goujaud2022pepit}
Goujaud, B., Moucer, C., Glineur, F., Hendrickx, J.~M., Taylor, A.~B., and
  Dieuleveut, A.
\newblock {PEPit}: computer-assisted worst-case analyses of first-order
  optimization methods in {Python}.
\newblock \emph{Mathematical Programming Computation}, 16\penalty0
  (3):\penalty0 337--367, 2024.

\bibitem[Gruntkowska et~al.(2025)Gruntkowska, Gaponov, Tovmasyan, and
  Richt{\'a}rik]{gruntkowska2025error}
Gruntkowska, K., Gaponov, A., Tovmasyan, Z., and Richt{\'a}rik, P.
\newblock Error feedback for muon and friends.
\newblock \emph{arXiv preprint arXiv:2510.00643}, 2025.

\bibitem[Harrane et~al.(2018)Harrane, Flamary, and
  Richard]{harrane_reducing_2018}
Harrane, I. E.~K., Flamary, R., and Richard, C.
\newblock On reducing the communication cost of the diffusion lms algorithm.
\newblock \emph{IEEE Transactions on Signal and Information Processing over
  Networks}, 5\penalty0 (1):\penalty0 100--112, 2018.

\bibitem[Inose \& Yasuda(2005)Inose and Yasuda]{inose2005unity}
Inose, H. and Yasuda, Y.
\newblock A unity bit coding method by negative feedback.
\newblock \emph{Proceedings of the IEEE}, 51\penalty0 (11):\penalty0
  1524--1535, 2005.

\bibitem[Ivkin et~al.(2019)Ivkin, Rothchild, Ullah, Braverman, Stoica, and
  Arora]{ivkin_communication-efficient_2019}
Ivkin, N., Rothchild, D., Ullah, E., Braverman, V., Stoica, I., and Arora, R.
\newblock Communication-efficient {Distributed} {SGD} with {Sketching}.
\newblock \emph{Advances in Neural Information Processing Systems (NeurIPS)},
  2019.

\bibitem[Kairouz et~al.(2019)Kairouz, McMahan, Avent, Bellet, Bennis, Bhagoji,
  Bonawitz, Charles, Cormode, Cummings, D'Oliveira, Rouayheb, Evans, Gardner,
  Garrett, Gascón, Ghazi, Gibbons, Gruteser, Harchaoui, He, He, Huo,
  Hutchinson, Hsu, Jaggi, Javidi, Joshi, Khodak, Konečný, Korolova,
  Koushanfar, Koyejo, Lepoint, Liu, Mittal, Mohri, Nock, Özgür, Pagh,
  Raykova, Qi, Ramage, Raskar, Song, Song, Stich, Sun, Suresh, Tramèr,
  Vepakomma, Wang, Xiong, Xu, Yang, Yu, Yu, and Zhao]{kairouz_advances_2019}
Kairouz, P., McMahan, H.~B., Avent, B., Bellet, A., Bennis, M., Bhagoji, A.~N.,
  Bonawitz, K., Charles, Z., Cormode, G., Cummings, R., D'Oliveira, R. G.~L.,
  Rouayheb, S.~E., Evans, D., Gardner, J., Garrett, Z., Gascón, A., Ghazi, B.,
  Gibbons, P.~B., Gruteser, M., Harchaoui, Z., He, C., He, L., Huo, Z.,
  Hutchinson, B., Hsu, J., Jaggi, M., Javidi, T., Joshi, G., Khodak, M.,
  Konečný, J., Korolova, A., Koushanfar, F., Koyejo, S., Lepoint, T., Liu,
  Y., Mittal, P., Mohri, M., Nock, R., Özgür, A., Pagh, R., Raykova, M., Qi,
  H., Ramage, D., Raskar, R., Song, D., Song, W., Stich, S.~U., Sun, Z.,
  Suresh, A.~T., Tramèr, F., Vepakomma, P., Wang, J., Xiong, L., Xu, Z., Yang,
  Q., Yu, F.~X., Yu, H., and Zhao, S.
\newblock Advances and {Open} {Problems} in {Federated} {Learning}.
\newblock \emph{arXiv:1912.04977 [cs, stat]}, December 2019.

\bibitem[Karimireddy et~al.(2019)Karimireddy, Rebjock, Stich, and
  Jaggi]{karimireddy_error_2019}
Karimireddy, S.~P., Rebjock, Q., Stich, S., and Jaggi, M.
\newblock Error {Feedback} {Fixes} {SignSGD} and other {Gradient} {Compression}
  {Schemes}.
\newblock In \emph{International {Conference} on {Machine} {Learning}
  ({ICML})}, 2019.

\bibitem[Karimireddy et~al.(2020)Karimireddy, Kale, Mohri, Reddi, Stich, and
  Suresh]{karimireddy_scaffold_2020}
Karimireddy, S.~P., Kale, S., Mohri, M., Reddi, S., Stich, S., and Suresh,
  A.~T.
\newblock Scaffold: Stochastic controlled averaging for federated learning.
\newblock In \emph{International Conference on Machine Learning (ICML)}, 2020.

\bibitem[Koloskova et~al.(2019)Koloskova, Stich, and
  Jaggi]{koloskova_decentralized_2019}
Koloskova, A., Stich, S., and Jaggi, M.
\newblock Decentralized {Stochastic} {Optimization} and {Gossip} {Algorithms}
  with {Compressed} {Communication}.
\newblock In \emph{International {Conference} on {Machine} {Learning}
  ({ICML})}, 2019.

\bibitem[Li \& Li(2022)Li and Li]{li2022analysis}
Li, X. and Li, P.
\newblock Analysis of error feedback in federated non-convex optimization with
  biased compression.
\newblock \emph{arXiv preprint arXiv:2211.14292}, 2022.

\bibitem[Ma et~al.(2018)Ma, Bassily, and Belkin]{ma2018power}
Ma, S., Bassily, R., and Belkin, M.
\newblock The power of interpolation: Understanding the effectiveness of sgd in
  modern over-parametrized learning.
\newblock In \emph{International Conference on Machine Learning}, pp.\
  3325--3334. PMLR, 2018.

\bibitem[McMahan et~al.(2017)McMahan, Moore, Ramage, Hampson, and
  Arcas]{mcmahan_communication-efficient_2017}
McMahan, B., Moore, E., Ramage, D., Hampson, S., and Arcas, B. A.~y.
\newblock Communication-{Efficient} {Learning} of {Deep} {Networks} from
  {Decentralized} {Data}.
\newblock In \emph{{International} {Conference} on {Artificial} {Intelligence}
  and {Statistics} {(AISTATS)}}, April 2017.

\bibitem[Mishchenko et~al.(2022)Mishchenko, Malinovsky, Stich, and
  Richt{\'a}rik]{mishchenko2022proxskip}
Mishchenko, K., Malinovsky, G., Stich, S., and Richt{\'a}rik, P.
\newblock Proxskip: Yes! local gradient steps provably lead to communication
  acceleration! finally!
\newblock In \emph{International Conference on Machine Learning}, pp.\
  15750--15769. PMLR, 2022.

\bibitem[{MOSEK ApS}(2025)]{mosek}
{MOSEK ApS}.
\newblock \emph{{MOSEK} Optimizer API for Python 11.0.21}, 2025.
\newblock URL \url{https://docs.mosek.com/11.0/pythonapi/index.html}.

\bibitem[Philippenko \& Dieuleveut(2020)Philippenko and
  Dieuleveut]{philippenko_artemis_2020}
Philippenko, C. and Dieuleveut, A.
\newblock Bidirectional compression in heterogeneous settings for distributed
  or federated learning with partial participation: tight convergence
  guarantees.
\newblock \emph{arXiv:2006.14591 [cs, stat]}, 2020.

\bibitem[Philippenko \& Dieuleveut(2021)Philippenko and
  Dieuleveut]{philippenko_preserved_2021}
Philippenko, C. and Dieuleveut, A.
\newblock Preserved central model for faster bidirectional compression in
  distributed settings.
\newblock \emph{{Advances} in {Neural} {Information} {Processing} {Systems}
  ({NeurIPS})}, 2021.

\bibitem[Redie et~al.(2026)Redie, Arablouei, and
  Werner]{redie2026sapefstepaheadpartialerror}
Redie, D.~K., Arablouei, R., and Werner, S.
\newblock Sa-pef: Step-ahead partial error feedback for efficient federated
  learning, 2026.
\newblock URL \url{https://arxiv.org/abs/2601.20738}.

\bibitem[Richt\'arik et~al.(2021)Richt\'arik, Sokolov, and
  Fatkhullin]{richtarik_ef21_2021}
Richt\'arik, P., Sokolov, I., and Fatkhullin, I.
\newblock {EF21}: {A} {New}, {Simpler}, {Theoretically} {Better}, and
  {Practically} {Faster} {Error} {Feedback}.
\newblock In \emph{Advances in {Neural} {Information} {Processing} {Systems}
  ({NeurIPS})}, 2021.

\bibitem[Seide et~al.(2014)Seide, Fu, Droppo, Li, and Yu]{seide_1-bit_2014}
Seide, F., Fu, H., Droppo, J., Li, G., and Yu, D.
\newblock 1-{Bit} {Stochastic} {Gradient} {Descent} and its {Application} to
  {Data}-{Parallel} {Distributed} {Training} of {Speech} {DNNs}.
\newblock In \emph{{Annual} {Conference} of the {International} {Speech}
  {Communication} {Association}}, 2014.

\bibitem[Stich \& Karimireddy(2020)Stich and
  Karimireddy]{stich_error-feedback_2020}
Stich, S.~U. and Karimireddy, S.~P.
\newblock The error-feedback framework: {Better} rates for {SGD} with delayed
  gradients and compressed updates.
\newblock \emph{Journal of Machine Learning Research}, 21:\penalty0 1--36,
  2020.

\bibitem[Stich et~al.(2018)Stich, Cordonnier, and Jaggi]{stich_sparsified_2018}
Stich, S.~U., Cordonnier, J.-B., and Jaggi, M.
\newblock Sparsified {SGD} with {Memory}.
\newblock In \emph{Advances in {Neural} {Information} {Processing} {Systems}
  ({NeurIPS})}. 2018.

\bibitem[Strom(2015)]{strom_scalable_2015}
Strom, N.
\newblock Scalable distributed {DNN} training using commodity {GPU} cloud
  computing.
\newblock In \emph{{Annual} {Conference} of the {International} {Speech}
  {Communication} {Association}}, 2015.

\bibitem[Sundararajan et~al.(2019)Sundararajan, Van~Scoy, and
  Lessard]{sundararajan2019canonical}
Sundararajan, A., Van~Scoy, B., and Lessard, L.
\newblock A canonical form for first-order distributed optimization algorithms.
\newblock In \emph{2019 American Control Conference (ACC)}, pp.\  4075--4080,
  2019.
\newblock \doi{10.23919/ACC.2019.8814838}.

\bibitem[Sundararajan et~al.(2020)Sundararajan, Van~Scoy, and
  Lessard]{sundararajan2020analysis}
Sundararajan, A., Van~Scoy, B., and Lessard, L.
\newblock Analysis and design of first-order distributed optimization
  algorithms over time-varying graphs.
\newblock \emph{IEEE Transactions on Control of Network Systems}, 7\penalty0
  (4):\penalty0 1597--1608, 2020.
\newblock \doi{10.1109/TCNS.2020.2988009}.

\bibitem[Tang et~al.(2021)Tang, Li, Liu, and Yan]{tang2021errorcompensatedx}
Tang, H., Li, Y., Liu, J., and Yan, M.
\newblock Errorcompensatedx: error compensation for variance reduced
  algorithms.
\newblock \emph{Advances in Neural Information Processing Systems},
  34:\penalty0 18102--18113, 2021.

\bibitem[Taylor et~al.(2018)Taylor, Van~Scoy, and Lessard]{taylor2018lyapunov}
Taylor, A., Van~Scoy, B., and Lessard, L.
\newblock {Lyapunov} {Functions} for {First}-{Order} {Methods}: {Tight}
  {Automated} {Convergence} {Guarantees}.
\newblock In \emph{International Conference on Machine Learning (ICML)}, 2018.

\bibitem[Taylor et~al.(2017{\natexlab{a}})Taylor, Hendrickx, and
  Glineur]{taylor2017performance}
Taylor, A.~B., Hendrickx, J.~M., and Glineur, F.
\newblock {Performance estimation toolbox (PESTO): automated worst-case
  analysis of first-order optimization methods}.
\newblock In \emph{Conference on Decision and Control (CDC)},
  2017{\natexlab{a}}.

\bibitem[Taylor et~al.(2017{\natexlab{b}})Taylor, Hendrickx, and
  Glineur]{taylor2017smooth}
Taylor, A.~B., Hendrickx, J.~M., and Glineur, F.
\newblock Smooth strongly convex interpolation and exact worst-case performance
  of first-order methods.
\newblock \emph{Mathematical Programming}, 161\penalty0 (1-2):\penalty0
  307--345, 2017{\natexlab{b}}.

\bibitem[Tian et~al.(2026)Tian, Li, Liu, and Shi]{tian2026ef21}
Tian, H., Li, X., Liu, S., and Shi, Y.
\newblock Ef21-rr: Fast o (1/t) rate for non-convex federated optimization with
  error feedback.
\newblock \emph{Automatica}, 183:\penalty0 112655, 2026.

\bibitem[Upadhyaya et~al.(2025)Upadhyaya, Banert, Taylor, and
  Giselsson]{upadhyaya2023automated}
Upadhyaya, M., Banert, S., Taylor, A.~B., and Giselsson, P.
\newblock {Automated} tight {Lyapunov} analysis for first-order methods.
\newblock \emph{Mathematical Programming}, 209\penalty0 (1):\penalty0 133--170,
  2025.

\bibitem[Vaswani et~al.(2019)Vaswani, Bach, and Schmidt]{vaswani2019fast}
Vaswani, S., Bach, F., and Schmidt, M.
\newblock Fast and faster convergence of sgd for over-parameterized models and
  an accelerated perceptron.
\newblock In \emph{The 22nd international conference on artificial intelligence
  and statistics}, pp.\  1195--1204. PMLR, 2019.

\bibitem[Vempala(2004)]{vempala2005random}
Vempala, S.~S.
\newblock \emph{The Random Projection Method}, volume~65 of \emph{DIMACS Series
  in Discrete Mathematics and Theoretical Computer Science}.
\newblock American Mathematical Society, 2004.
\newblock ISBN 978-0-8218-3793-1.

\bibitem[Wu et~al.(2018)Wu, Huang, Huang, and Zhang]{wu_error_2018}
Wu, J., Huang, W., Huang, J., and Zhang, T.
\newblock Error {Compensated} {Quantized} {SGD} and its {Applications} to
  {Large}-scale {Distributed} {Optimization}.
\newblock In \emph{International {Conference} on {Machine} {Learning}
  ({ICML})}, 2018.

\bibitem[Yamashita et~al.(2010)Yamashita, Fujisawa, Nakata, Kojima, Kobayashi,
  and Goto]{yamashita2010sdpa7}
Yamashita, M., Fujisawa, K., Nakata, K., Kojima, M., Kobayashi, K., and Goto,
  K.
\newblock A high-performance software package for semidefinite programs: {SDPA}
  7.
\newblock Technical Report B-460, Department of Mathematical and Computing
  Science, Tokyo Institute of Technology, 2010.

\bibitem[Zheng et~al.(2019)Zheng, Huang, and
  Kwok]{zheng_communication-efficient_2019}
Zheng, S., Huang, Z., and Kwok, J.
\newblock Communication-{Efficient} {Distributed} {Blockwise} {Momentum} {SGD}
  with {Error}-{Feedback}.
\newblock In \emph{Advances in {Neural} {Information} {Processing} {Systems}
  ({NeurIPS})}, 2019.

\end{thebibliography}
